%% file: main.tex
\documentclass[conference]{IEEEtran}
\IEEEoverridecommandlockouts
\usepackage{cite}
\usepackage{amsmath,amssymb,amsfonts}
\usepackage{amsthm}
\usepackage{algorithmic}
\usepackage{graphicx}
\usepackage{textcomp}
\usepackage{xcolor}
\usepackage{subcaption}
\usepackage{enumerate}
\usepackage[colorlinks=true, linkcolor=blue, citecolor=blue, urlcolor=blue]{hyperref}
\def\BibTeX{{\rm B\kern-.05em{\sc i\kern-.025em b}\kern-.08em
    T\kern-.1667em\lower.7ex\hbox{E}\kern-.125emX}}

\newtheorem{lemma}{Lemma}
\newtheorem*{lemma*}{Lemma}

\newif\ifanonymized
\anonymizedfalse
\usepackage{pgfplots}
\pgfplotsset{compat=1.18}
\usepackage{tikz}
\usetikzlibrary{intersections,angles,quotes,calc,arrows.meta,positioning,spy}
\newlength{\capstanunit}
\newlength{\vineunit}\newlength{\bodyouter}\newlength{\bodyinner}

\input{preambles/2d-maze-preamble}

\input{preambles/scaling-preamble}

\begin{document}

\title{An Efficient Algorithm for Minimum-Pressure Growth Planning of Vine Robots}

\ifanonymized
    \author{Author names omitted for anonymous review.}
\else
\author{
    Andres C.\ Torres$^1$, Tobia Marcucci$^{2,1}$, and Elliot W.\ Hawkes$^1$
    \thanks{$^{1}$Department of Mechanical Engineering, University of California
        Santa Barbara, CA 93106, USA}
    \thanks{$^{2}$Department of Electrical and Computer Engineering, University of
            California Santa Barbara, CA 93106, USA}
    % \thanks{This work was supported by the Air Force Office of Scientific Research award FA9550-25-1-0375.}
    \thanks{Emails: \texttt{\{actorres,marcucci,ewhawkes\}@ucsb.edu}}%
}
\fi

\newcommand{\algoname}{\texttt{\ifanonymized AnonymousPackage\else VinePlanner\fi}}

\maketitle

\begin{abstract}
    % everyone's verison
    Vine robots navigate cluttered environments by extending from their tip. Although their ability to operate in such environments has been extensively demonstrated, little work has addressed growth planning, i.e., finding optimal growth paths. Moreover, existing planners do not account for the growth pressure necessary to follow a given path, which can cause the robot to burst when it is too high. In this paper, we address the problem of finding minimum-pressure paths for vine robots growing around polytopic obstacles.  We propose an efficient algorithm that is guaranteed to find globally optimal solutions in 2D and approximate solutions in 3D, with an error that vanishes as a discretization parameter approaches zero.  First, we derive a growth pressure equation for vine robots of arbitrary shape, which we use to show that there always exists a minimum-pressure path that is piecewise-linear and can bend only at specific points on the obstacles. We then leverage this observation to reduce the growth-planning problem to a shortest-path problem with time-dependent weights, which we efficiently solve using a modified Dijkstra's algorithm.  We demonstrate the speed and scalability of our approach through numerical simulations.  We also validate our algorithm with hardware experiments and provide an open-source and high-performance implementation in the Python package \algoname.
\end{abstract}

\section{Introduction}
Soft robots are made of compliant materials, and leverage the compliance of their bodies for performance benefits~\cite{yasaoverview2023}.  One class of such robots is soft growing robots, or vine robots, which are thin-walled, cylindrical, pneumatic robots. When pressurized, a vine robot everts and lengthens from its tip~\cite{HawkesVine2017}.
Vine robots have numerous applications, including industrial inspection~\cite{HeapInspection2025}, exploration of ancient ruins~\cite{CoadRuins2020}, search and rescue~\cite{McFarlandSearch2024}, and medical applications~\cite{haggertyIntubation2025, livine2021, davymagnetic2025, KimEndoscope2026}. In many of these applications, the robot must navigate around obstacles to reach a goal. We refer to the process of finding such a path as \emph{growth planning}. Along a given path, the growth pressure increases as the robot extends or changes direction~\cite{blumenscheinmodeling2017}. If the pressure exceeds the burst pressure of the robot, the extending body can rupture. Consequently, only a small subset of paths may allow the robot to reach the goal while maintaining a safe pressure, as illustrated in Fig.~\ref{fig:2D-maze}.

\begin{figure}[t]
  \centering
  \begin{subfigure}[b]{0.5060\linewidth}
    \centering\input{figures/2D-maze-paths}%
    % Trim below the obstacle course so its frame sits on the same baseline as the photo
    % in (a). The tikzpicture's box is taller than the drawn frame (padding for
    % the corner markers, plus the line the picture sits on), so this pulls the
    % caption up to compensate. Verified by measuring the rendered PDF;
    % each 0.01 here moves the frame about 0.017 in relative to the photo.
    \vspace{-0.112\linewidth}
    \caption{}\label{fig:2D-maze-sim}
  \end{subfigure}\hfill
  \begin{subfigure}[b]{0.4690\linewidth}
    \centering
    \includegraphics[width=\linewidth]{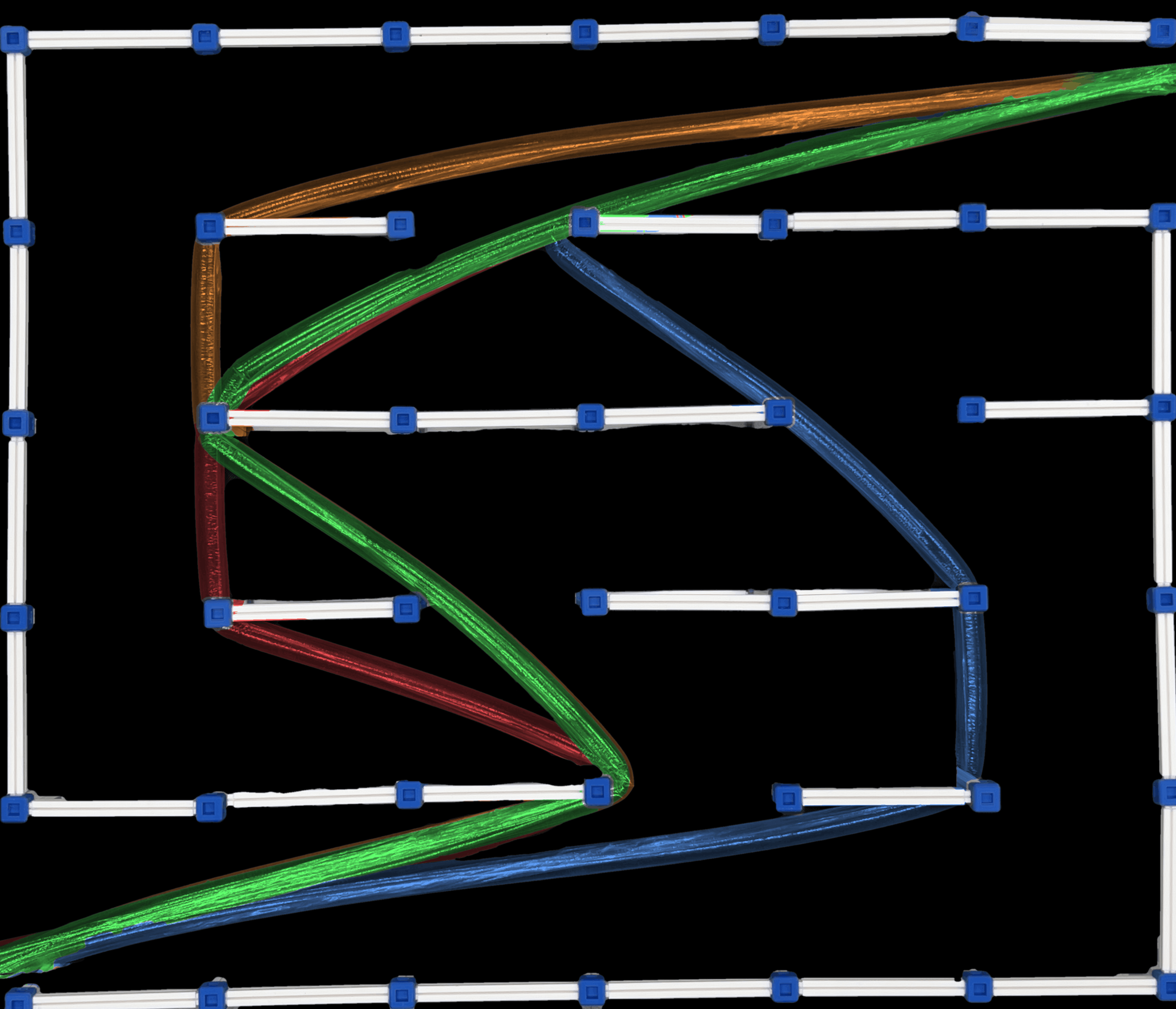}
    \caption{}\label{fig:2D-maze-hw}
  \end{subfigure}\\[6pt]
  \begin{subfigure}{\linewidth}
    \centering\input{figures/2D-maze-pressure}
    \caption{}\label{fig:maze-pressure}
  \end{subfigure}
  \caption{A 2D obstacle course with four near-optimal paths. (a) The paths are first determined from our model in simulation.  (b) The same four paths are grown in hardware and overlaid into a single image.  (c) Predicted growth pressure versus deployed length. With a hypothetical burst pressure of 25~kPa (dashed line), only the minimum-pressure path (green) is physically feasible.}
  \label{fig:2D-maze}
\end{figure}

In this paper, we propose a computationally efficient growth planner that finds globally optimal minimum-pressure paths in 2D environments.  Additionally, we provide an extension to 3D environments that produces paths optimal up to a discretization parameter.  

\subsection{Related work}
Path planning is a deeply studied problem for conventional rigid robots. Classical planners seek paths that minimize distance, traversal time, or the magnitude of trajectory derivatives (see, e.g.,~\cite{lavalle2006planning}), which are natural objectives for rigid robots.
Some of these objectives, such as minimum-distance, can also be applied to planning the growth of a vine robot. However, as demonstrated in Fig.~\ref{fig:dense-maze}, such objectives can result in disastrously high pressures, since growth pressure is not determined solely by path length, but is often dominated by the total path angle, as explained in Section~\ref{sec:growth_eq}.

Existing works on growth planning for vine robots have focused on objectives other than growth pressure. The planning algorithm in~\cite{greerrobust2020} minimizes uncertainty in reaching the goal. To achieve this, the robot is shaped one turn at a time by enumerating several candidate turns and selecting the one that is most likely to reach the next waypoint despite uncertainty. The approach in~\cite{GaoParallel2026} builds on~\cite{greerrobust2020} by incorporating a vine robot simulation with the Stable Sparse-RRT* algorithm~\cite{LiSST2016} to find paths that minimize the number of turns that must be manufactured in the robot.  They assume constant growth pressure along the path, which may not be appropriate in some applications.

\subsection{Contributions}
This work has three main contributions:
\begin{itemize}
\item A novel growth pressure equation that applies to the shape of any vine robot.
\item An efficient growth-planning algorithm that minimizes growth pressure. This method is guaranteed to identify a globally optimal path in 2D, and unlike previous methods it is also applicable in 3D.
\item An open-source and high-performance implementation of our method in the Python package \algoname.
\end{itemize}

\section{Growth Pressure Equation}
\label{sec:growth_eq}
In this section, we start from a previous growth pressure equation and derive a generalized equation to handle a vine robot made of multiple straight segments and turns.  We then extend this equation to calculate the growth pressure of any vine robot shape, which we represent as a piecewise-smooth path.

\subsection{Previous growth pressure equation}
Blumenschein et al.~\cite{blumenscheinmodeling2017} model the growth pressure of a vine robot along a straight path as
\begin{equation}
    P = Y+(T+\lambda L)/A, \label{eq:laura-straight}
\end{equation}
and around a single turn of constant curvature as
\begin{equation}
    P = Y+Te^{\mu\theta} / A. \label{eq:laura-curve}
\end{equation}
The symbols in these equations represent the following quantities:
\begin{itemize}
    \item $P\in\mathbb{R}_+$ is the growth pressure,
    \item $A\in\mathbb{R}_+$ is the cross-sectional area of the vine robot,
    \item $Y\in\mathbb{R}_+$ is the yield pressure required to evert the vine,
    \item $\lambda\in\mathbb{R}_+$ is a length-dependent friction term (note we combine several terms from~\cite{blumenscheinmodeling2017} into this single term); this term accounts for the friction force in straight sections between the inside of the vine robot body and the ``tail", or material passing through the core of the body to evert at the tip,
    \item $L\in\mathbb{R}_+$ is the total length of the vine robot,
    \item $T\in\mathbb{R}_+$ is the tension force in the tail caused by the tug of the reel,
    \item $\mu\in\mathbb{R}_+$ is the Capstan friction coefficient, which accounts for the friction between the tail and the inside of the body in a bend, analogous to a rope around a cylinder~\cite{beervector2019},
    \item $\theta\in[0,\pi]$ is the angle of the turn measured in radians.
\end{itemize}
Here we omit the velocity term found in the original equations, since we assume slow growth.

Blumenschein et al.~\cite{blumenscheinmodeling2017} also propose a version of the equation that combines~\eqref{eq:laura-straight} and~\eqref{eq:laura-curve} into a single equation that accounts for multiple straight segments.
However, this equation does not accurately model the coupling of length and bends nor the multiplicative accumulation of friction in the Capstan friction model.  We also note that~\cite{HeapInspection2025} used a numerical model to account for multiple straight sections and a single turn, but did not formalize the model with equations.  Below, we derive a growth pressure equation that can accurately model a vine robot with multiple straight segments.

\subsection{Proposed growth pressure equation}
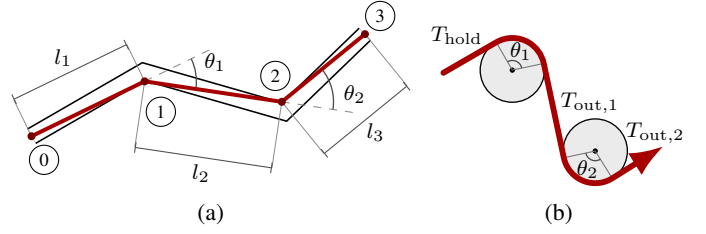
\begin{figure}
    \centering
    \input{figures/simple-vine}
    \hfill
    \input{figures/two-cylinders}%
    \caption{Sequential turning has a multiplicative effect on tension. (a) Simple vine robot with three segments and two turns. (b) Applying the Capstan equation to two consecutive angles, we obtain $T_{\mathrm{out},2}=T_{\mathrm{hold}}e^{\mu\theta_1}e^{\mu\theta_2}$; each turn thus contributes multiplicatively.}
\end{figure}

Consider the vine robot with two turns in Fig.~\ref{fig:simple-vine}, with three straight segments of lengths $ l_1, l_2, l_3$ and two turns with angles $\theta_1,\theta_2$.
We calculate the tension before and after each point, denoted by $\tau_i^-$ and $\tau_i^+$ for $i=0,1,2,3$.
The initial tension at point~0 is $\tau_0^+ =T$.
The tension immediately before the turn at point~1 is given by
\begin{equation*}
    \tau_1^-= \tau_0^+ + \lambda l_1 = T + \lambda l_1.
\end{equation*}
According to the Capstan equation, as illustrated in Fig.~\ref{fig:capstan}, to calculate the tension $\tau_1^+$ after the first turn, we multiply the accumulated tension $\tau_1^-$ by an exponential function of the angle: 
\begin{equation*}    
    \tau_1^+ = \tau_1^- e^{\mu\theta_1} = (T + \lambda l_1)e^{\mu\theta_1}.
\end{equation*}
As the vine robot everts towards point~2, the tension increases due to the friction from the length term:
\begin{equation*}
    \tau_2^- = \tau_1^+ +\lambda l_2 = (T + \lambda l_1)e^{\mu\theta_1}+\lambda l_2.
\end{equation*}
Similarly to above, after the angle at point~2, we multiply the tension $\tau_2^-$ by an exponential:
\begin{equation*}
    \tau_2^+ = \tau_2^- e^{\mu\theta_2} = \Big((T + \lambda l_1)e^{\mu\theta_1}+\lambda l_2\Big)e^{\mu\theta_2}.
\end{equation*}
The tension before point~3 is again obtained by adding the friction from the increase in length:
\begin{equation*}
    \tau_3^- = \tau_2^+ + \lambda l_3 = \Big((T + \lambda l_1)e^{\mu\theta_1}+\lambda l_2\Big)e^{\mu\theta_2}+\lambda l_3.
\end{equation*}

Extending this to a vine robot with $n$ segments, we obtain the following recursive equation for the pressure $P$ at the end of the robot:
\begin{subequations}
\begin{align}
  \tau_0^+ &= T, \\
  \tau_i^+ &= \bigl(\tau_{i-1}^+ + \lambda l_i\bigr)\,e^{\mu\theta_i},
    \quad i = 1,\dots,n-1, \label{eq:pressure-recur} \\
  P &= Y + (\tau_{n-1}^+ +\lambda l_n)/A, \label{eq:pressure-eq}
\end{align}
\end{subequations}
where $ l_i\in\mathbb{R}_+$ is the $i$-th segment length, $\theta_i\in[0,\pi]$ is the angle between the $i$-th segment and $(i+1)$-th segment, and with all other symbols retaining their original meaning.

We note that \eqref{eq:pressure-eq} is restricted to vine robots whose shapes are piecewise-linear paths.  Additionally, because the velocity term from~\cite{blumenscheinmodeling2017} is path-independent, it can be reintroduced if velocity cannot be considered to be negligible.  In the next section, we derive a more general formula for piecewise-smooth paths.

\subsection{Extension to piecewise-smooth paths}
We extend~\eqref{eq:pressure-eq} to calculate the pressure of any vine robot shape that we represent as a piecewise-smooth curved path.

We first consider a twice-differentiable path parameterized by Euclidean arc length $\gamma:[0,L]\to\mathbb{R}^d$, where $d\in\{2,3\}$.  We can approximate this path as arbitrarily many straight segments.  For small $h$, we can approximate~\eqref{eq:pressure-recur} as follows:
\begin{equation*}
    \tau(s+h)=\left(\tau(s)+\lambda h\right)e^{\mu\theta(s,s+h)},
\end{equation*}
where $\tau(s)$ is the previously accumulated tail tension and
\begin{equation}
\theta(a, b)=\int_a^{b}\|\gamma''(\sigma)\|\,d\sigma \label{eq:angle-segment}
\end{equation}
denotes the angle accumulated between $a$ and $b$.
Using the definition of the derivative, algebraic manipulation, and the exponential limit identity, we determine that the derivative of the tension is
\begin{equation}
    \frac{d\tau}{ds}=\tau(s)\mu\|\gamma''(s)\|+\lambda. \label{eq:pressure-differential}
\end{equation}
With initial condition $\tau(0)=T$, we can solve the differential equation \eqref{eq:pressure-differential} using the integrating factor $e^{-\mu\theta(0,s)}$ to obtain
\begin{equation}
    \tau(s)=Te^{\mu\theta(0, s)}+\lambda\int_0^se^{\mu(\theta(0,s)-\theta(0,t))}\,dt. \label{eq:tension-int-single-path}
\end{equation}

We can further generalize the steps above to a piecewise twice-differentiable path $\gamma:[0,L]\to\mathbb{R}^d$ parameterized by arc length.
We partition $\gamma$ into $n$ segments, divided by the breakpoints (where the path derivative is undefined, i.e., corners or kinks) $s_1<\dots<s_{n-1}$, with $s_0=0$ and $s_n=L$.
We let $\Theta(s)$ be the accumulated angle at point $s$ along the path by computing the accumulated angle along the smooth segments and adding the angles formed at each breakpoint.  
For $s\in[s_{j-1},s_j]$, we have
\begin{equation*}
    \Theta(s)=\sum_{i=1}^{j-1}\left(\theta(s_{i-1},s_i)+\phi_i\right)+\theta(s_{j-1},s), \label{eq:angle}
\end{equation*}
where the breakpoint angles are
\begin{equation*}
    \phi_i=\arccos (\gamma'(s_i^-) \cdot \gamma'(s_i^+) ) \in[0,\pi],
\end{equation*}
for $i=1,\dots,n-1$.
Note that $\Theta(s)$ is unsigned in the sense that reversals accumulate rather than cancel.
Substituting $\Theta$ for $\theta$ in the tail tension equation \eqref{eq:tension-int-single-path} and adding the yield pressure term, we obtain the general pressure equation
\begin{equation}
    P(s)=Y+\left(Te^{\mu\Theta(s)}+\lambda\int_0^se^{\mu(\Theta(s)-\Theta(t))}\,dt\right)/A. \label{eq:pressure-int}
\end{equation}

\section{Problem Statement} \label{problem-statement}
We first describe the problem informally.  Consider two points in 2D or 3D space and assume there are finitely many polytopic obstacles in the space that cannot be traversed.  We aim to find a path that has the minimum growth pressure between the two points without entering the interior of any of the obstacles.

We represent the vine robot as a piecewise twice-differentiable path parameterized by Euclidean arc length $\gamma:[0,L]\to\mathbb{R}^d$.
We let $\mathcal O_1,\dots,\mathcal O_m$ be pairwise-disjoint obstacles represented as closed polytopes (i.e., bounded convex polyhedra).
We consider two points $x_{\text{init}}, x_{\text{final}}\in\mathbb{R}^d$, such that $x_{\text{init}}, x_{\text{final}}\notin \mathrm{int}(\mathcal O_j)$, for $j=1,\ldots,m$.

Our goal is to find the curve $\gamma$ that connects $x_{\text{init}}$ to $x_{\text{final}}$ and minimizes the growth pressure at the endpoint of the vine robot:
\begin{subequations}
\label{eq:planning_problem}
\begin{align}
\mathrm{minimize} \quad & P(L) \label{eq:opt} \\
\mathrm{subject \ to} \quad
& \gamma(0)=x_{\mathrm{init}}, \\
& \gamma(L)=x_{\mathrm{final}}, \\
& \gamma(s) \notin \mathrm{int}(\mathcal O_j), && s \in [0,L], \ j=1,\ldots,m,
\end{align}
\end{subequations}
where the path $\gamma$ and its length $L$ are the optimization variables and $P$ is given by \eqref{eq:pressure-int}.

\section{Growth Planning Algorithm} \label{growth-planning}
In this section, we describe our algorithm for solving problem~\eqref{eq:planning_problem} in 2D and 3D.
This algorithm is guaranteed to find a globally optimal solution in the 2D case.  In 3D, the path found is optimal only up to a discretization error.  We also provide a brief discussion of the algorithm's time complexity.

We start with a lemma that is at the core of our algorithm.
Recall that a ridge of a polytope is the intersection of two distinct facets (a vertex in 2D and an edge in 3D).

% We start by defining a facet and a ridge. Then we present Lemma~\ref{lem:ridges} that is at the base of our algorithm.
% \begin{definition}[Facet]
% A closed polytope is the intersection of finitely many closed halfspaces, none of which is redundant.  A facet of an obstacle is the portion of its boundary lying in the bounding hyperplane of one such halfspace (an edge in 2D and a face in 3D).
% \end{definition}

% \begin{definition}[Ridge]
%     A ridge of an obstacle is the intersection of two of its facets (a vertex in 2D and an edge or vertex in 3D).
% \end{definition}

\begin{lemma} \label{lem:ridges}
    If problem~\eqref{eq:planning_problem} is feasible, then it admits a globally optimal solution that is piecewise-linear with breakpoints exclusively on obstacle ridges.
\end{lemma}

This allows us to limit our search exclusively to piecewise-linear solutions.  In 2D, the ridges of the finitely many obstacles are finitely many vertices, so by Lemma~\ref{lem:ridges} the search can be restricted to piecewise-linear paths whose breakpoints lie in a finite set of points; since revisiting a vertex can only increase tension, it suffices to consider paths with no cycles.  Therefore, we can apply similar techniques to those in the minimum-distance path-planning problem~\cite{lozanoalgorithm1979}, which also relies on this fact.  We start by constructing a visibility graph with obstacle vertices.  Next, we construct the line graph of the visibility graph because the angle between segments cannot be recovered from the visibility graph.  Lastly, we apply a modified Dijkstra's algorithm that handles the non-additive weights on edges of the line graph.

\subsection{Line graph}
We start by constructing a visibility graph $G=(\mathcal{V}, \mathcal{E})$, where $\mathcal{V}$ is the set of obstacle vertices (plus the start and end points) and $\mathcal{E}$ are the graph edges.  Recall that a visibility graph is a graph where each vertex is a point and two points have an edge between them if the line segment connecting them does not pass through the interior of any obstacles~\cite[Chapter~15]{decomputational2008}.  Next, we construct the directed line graph $L=(\mathcal{V}_L,\mathcal{E}_L)$ of the visibility graph $G$.  The vertices of $L$ are the edges of $G$, i.e., $\mathcal{V}_L=\mathcal{E}$.  Two vertices of $L$ are connected by an edge in the line graph if the head of the first edge coincides with the tail of the second edge in the visibility graph.  For example, assume that $u,v,w\in\mathcal{V}$ and $(u,v),(v,w)\in\mathcal{E}$ are vertices and edges in the visibility graph $G$.  Since the head of $(u,v)$ coincides with the tail of $(v,w)$, $(u,v,w)$ is an edge in $L$.
Observe that the edges of the line graph correspond to physical turns of the vine robot.
This will allow us to formulate a shortest-path problem where traversing an edge has a cost that depends on the angle turned by the robot, as well as on the path length.

We also add a source vertex $s$ to the line graph $L$ with an edge to every other vertex of the form $(x_\text{init}, q)$.
Similarly, we add a target vertex $t$ and connect it to every other vertex of the form $(p, x_\text{final})$.

\subsection{Shortest-path problem}
We now formulate a shortest-path problem over the line graph $L$ that models problem~\eqref{eq:planning_problem} exactly.
The path in the line graph must start at vertex $s$ and terminate at vertex $t$.
The edges leaving the source vertex $s$ carry the initial tension $T$.
The edge entering the target vertex $t$ from each vertex $(p,x_{\mathrm{final}})$ has cost $\lambda$ times the length of $x_{\mathrm{final}}-p$.
All the other edges $(u,v,w)$ have cost equal to the corresponding increase in tension (recall the recursive formula~\eqref{eq:pressure-recur}):
\begin{equation}
    c_{uvw}(\tau_u^+) = \tau_u^+ (e^{\mu\theta_{uvw}} - 1) + \lambda l_{uv} e^{\mu\theta_{uvw}} \label{eq:edge_weights}
\end{equation}
where $\tau_u^+$ is the tension immediately after point $u$, $l_{uv}$ is the length of vector $v - u$, and $\theta_{uvw} \in [0, \pi]$ is the angle between the vectors $v - u$ and $w - v$.
A shortest path in this weighted line graph corresponds to a minimum-pressure growth of the vine robot.
The growth pressure is equal to the optimal value of this shortest-path problem, divided by the area $A$ and added to the yield pressure $Y$ as in~\eqref{eq:pressure-eq}.

\subsection{Modified Dijkstra's algorithm}
Because the edge weights~\eqref{eq:edge_weights} are calculated recursively, the tension does not increase additively.  Therefore, a standard implementation of Dijkstra's algorithm~\cite{dijkstraalgo1959} does not apply.  However, a more general version of Dijkstra's algorithm applies to the time-dependent shortest-path problem, which admits a more generic cost function.

In the time-dependent shortest path problem~\cite{kaufmanfastest1993}, each edge has a function that calculates the edge cost dependent on the accumulated cost.  The objective is to find the path that has the minimum cost through the graph.  Finding the minimum-pressure path is an instance of this problem, with tension in place of time.

Kaufman and Smith~\cite{kaufmanfastest1993} establish two conditions for Dijkstra's algorithm to find an optimal solution of a time-dependent shortest-path problem:
\begin{itemize}
    \item Consistency~\cite[Theorem~2]{kaufmanfastest1993}: If $\tau_1\le\tau_2$, then the edge costs~\eqref{eq:edge_weights} must satisfy $\tau_1+c_{uvw}(\tau_1)\le\tau_2+c_{uvw}(\tau_2)$.  Here this reduces to $\tau_1e^{\mu\theta_{uvw}}\le\tau_2e^{\mu\theta_{uvw}}$, which holds because $e^{\mu\theta_{uvw}}>0$.
    \item Nonnegativity~\cite[Theorem~4]{kaufmanfastest1993}: Edge costs~\eqref{eq:edge_weights} are nonnegative.  This follows because $\mu,\theta_{uvw}\ge0$ imply $e^{\mu\theta_{uvw}}\ge1$, and $T, \lambda, \tau_u^+, l_{uv} \geq 0$.
\end{itemize}

After finding the optimal path in the line graph with the general version of Dijkstra's algorithm, the robot growth can be reconstructed by connecting the relevant vertices of the visibility graph $G$.  By Lemma~\ref{lem:ridges}, and the correctness of the label-setting method~\cite{kaufmanfastest1993}, this path is a globally optimal solution of problem~\eqref{eq:planning_problem}.

\subsection{Extension to 3D}
As noted in previous work~\cite{lozanoalgorithm1979, papadimitrioualgorithm1985}, the minimum-distance path-planning problem in 3D does not reduce to a finite set of piecewise-linear paths.  This is because in 3D, obstacle ridges are edges that contain infinitely many points.  So the search space remains infinite.  This same issue also applies to the minimum-pressure path.

Thus, in order to reduce the infinite search space to a finite one, we take inspiration from previous work~\cite{lozanoalgorithm1979, papadimitrioualgorithm1985} and discretize edges of obstacles and add them to the vertex set.  The vertex set now becomes $\mathcal{V}=\{x_{\text{init}},x_{\text{final}}\}\cup\mathcal{V}_{\text{vertices}}\cup \mathcal{D}$, where $\mathcal{V}_\text{vertices}$ is the set of vertices of polytopic obstacles and $\mathcal{D}$ is the set of discretization points on obstacle edges.  The choice of discretization does not affect the rest of the algorithm.  For simplicity, in our implementation we choose a uniform discretization where points are evenly spaced along the edges.  The updated vertex set $\mathcal{V}$ is the only change to the algorithm and the rest is exactly the same.

This approach yields low-pressure paths that are close to a global optimum, with the suboptimality gap vanishing as the discretization gets finer and finer.

\subsection{Time complexity} \label{time-complexity}
We briefly comment on the time complexity of the proposed algorithm.
Let $n = |\mathcal{V}|$ denote the number of vertices in the visibility graph, i.e., the total number of obstacle vertices, plus the start and end point, in the 2D environment.
A straightforward construction of the visibility graph requires $O(n^3)$ operations, and the resulting line graph has at most $O(n^2)$ vertices and $O(n^3)$ edges.
Finally, the modified Dijkstra's algorithm requires $O(n^3 \log n)$ operations~\cite{kaufmanfastest1993}.
Despite these worst-case complexity bounds, we will see in the experiments that the proposed algorithm exhibits substantially better scaling in practice.

\section{Experiments} \label{experiments}
In this section, we evaluate our approach through simulation and hardware experiments.
Every experiment was run using our Python library \algoname, with source code freely available at \url{\ifanonymized https://anonymous.4open.science/r/AnonymousGrowthPlanner/\else https://github.com/Ahsoka/VinePlanner\fi}. The library is optimized for performance and uses Numba~\cite{lamnumba2015} to parallelize the construction of the visibility graph. All simulation experiments were performed on a Windows 10 machine equipped with an 8-core AMD Ryzen 7 7800X3D processor and 64 GB of DDR5 RAM.

\subsection{Runtime scaling study}
In this section, we study empirically the scalability of our algorithm, analyzed in the worst case in Section~\ref{time-complexity}.  We randomly generated 2D obstacle courses with obstacle counts spaced roughly logarithmically from 100 to 200,000.  The resulting wall-clock times are reported in Fig.~\ref{fig:runtime-scaling} (we do not include the time needed to compile the Numba functions, since the compiled functions can be reused in each test). 
The figure shows that, for courses with hundreds of obstacles, the algorithm runs in a few tens of milliseconds. We also push the algorithm to its limits by considering courses with hundreds of thousands of obstacles, for which the runtime increases to the order of one hour.
The results show that, in practice, materializing the visibility graph dominates Dijkstra's algorithm (the search terminates once the goal is reached, so only a fraction of the line graph is explored).  In this experiment, our algorithm has an empirical time complexity of approximately $O(n^2)$.

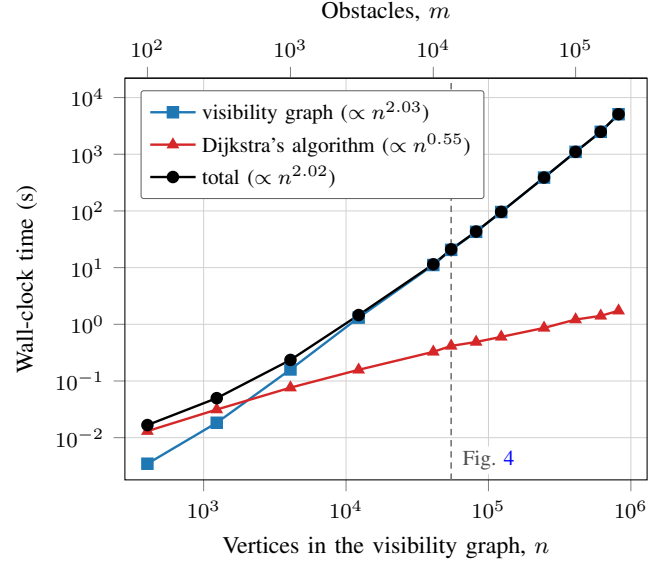
\begin{figure}
  \centering
  \input{figures/runtime-scaling}
  \caption{Runtime scaling of our growth planning algorithm tested on randomly generated 2D obstacle courses.  Top $x$-axis: number of obstacles, bottom $x$-axis: the number of vertices in the visibility graph, $y$-axis: wall-clock time in seconds. The data points along the gray dashed line correspond to Fig.~\ref{fig:dense-maze}.}
  \label{fig:runtime-scaling}
\end{figure}

To give a sense of the complexity of the obstacle courses in our experiment, in Fig.~\ref{fig:dense-maze} we report the obstacle course corresponding to 15,000 obstacles (54,712 vertices) in Fig.~\ref{fig:runtime-scaling}.  The entire algorithm described in Section~\ref{growth-planning} runs in 21~seconds for this problem.
For the same problem, the minimum-distance path accumulates roughly twice the total path angle, causing the growth pressure to explode exponentially to 19,430~kPa, over 60 times the 318~kPa of our path, and far beyond the burst pressure of any vine robot, in exchange for a 4\% reduction in length.  We do not benchmark against~\cite{greerrobust2020,GaoParallel2026} directly, as they
optimize uncertainty and manufacturing complexity rather than pressure and are not designed for environments of this scale.

\begin{figure}
    \centering
    \input{figures/dense-maze-spy}
    \caption{Our growth planning algorithm applied to a dense obstacle course with 15,000 obstacles.  The visibility graph has 54,712 vertices and the line graph has 801,008 vertices and 12,738,916 edges.  Our algorithm solves this problem in 21 seconds.  These are the data points in Fig.~\ref{fig:runtime-scaling} indicated by the gray dashed line.}
    \label{fig:dense-maze}
\end{figure}
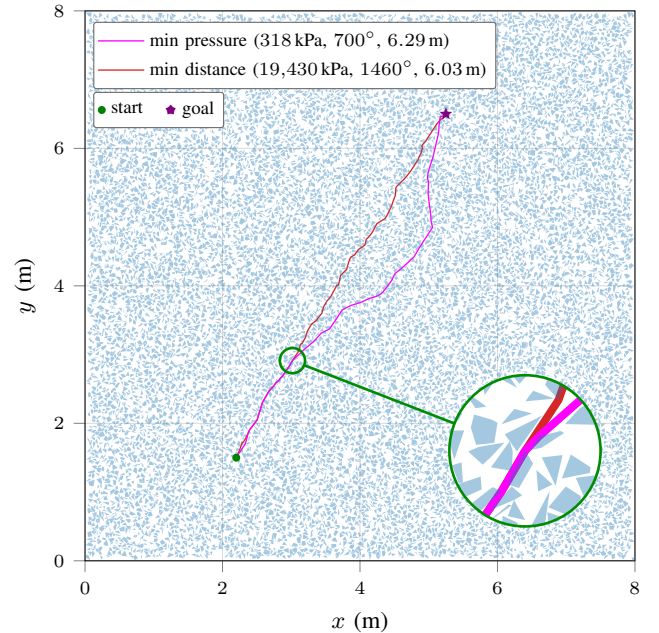

\subsection{Hardware experiments}
For our hardware experiments, we constructed a vine robot from 1.5 inch layflat (24 mm diameter), 2 mil (0.05 mm thick) LDPE polytubing (Uline S-3521).  We use a simple reel design inside an airtight PVC tee fitting and grow the vine robot through the 90$^\circ$ outlet.  We utilize a custom-designed, 3D-printed modular obstacle course made up of voxels and laser-cut acrylic tiles that can be inserted on any of the six faces of the voxels to serve as obstacles.

In our model of this vine robot, we use the following constants:
\begin{itemize}
    \item The yield pressure $Y=1.917\,$kPa, determined by measuring the pressure of an extremely short vine robot (to minimize the effects of the length term) not connected to the reel.
    \item The length-dependent friction term $\lambda/A=0.771\,$kPa/m, determined using the motorized experiment setup described in~\cite[Path-Dependent Losses]{blumenscheinmodeling2017}.
    \item The tension $T/A=5.426\,$kPa, determined by measuring the pressure required to grow an extremely short vine robot wrapped around the reel and subtracting $Y$.
    \item The Capstan friction coefficient $\mu=0.315$, determined by growing several vine robots with multiple turns of fixed angle, measuring the pressure before and after each turn, and fitting the best $\mu$.
\end{itemize}

\subsubsection{Two-dimensional obstacle course}
To validate that the pressure model ranks paths correctly (which is what the growth planner relies on), we designed a simple obstacle course, as shown in Fig.~\ref{fig:2D-maze}, where we can enumerate all possible cycle-free piecewise-linear paths that have breakpoints at obstacle vertices.  We computed the pressure of each such path using \eqref{eq:pressure-eq}, selected the four minimum-pressure paths mutually separated by at least 1\,kPa (below this, the paths are difficult to distinguish during experimental testing, and there is little practical reason to rank them), and ranked them from lowest (1) to highest (4).  We then manually placed the vine robot in the shape the model predicted to be lowest.  We then grew the robot in the final section of the obstacle course to record peak pressure.  Next, we manually moved the robot into the next configuration, and grew it again in the final section, recording peak pressure.  We repeated this for the final two paths, and then ranked them from lowest to highest pressure.  We then reset the robot, winding it back onto its spool, and repeated this six times, always testing paths in the same order.  Our model predicted the correct order 100\% of the time.  Across the six trials, the absolute pressures shifted substantially from trial to trial ($\approx\pm$7\,kPa), but within any single trial all four paths shifted together, so the order was always maintained.

\subsubsection{Three-dimensional obstacle course}
We repeated the experiment described above with a simple 3D obstacle course, depicted in Fig.~\ref{fig:3D-maze}.  For simplicity, we added two uniformly spaced points per obstacle edge.  We again found the four paths with predicted maximum pressures that were at least 1\,kPa different.  We ran six trials and once again saw inter-trial pressure variation, but obtained the correct order 100\% of the time.

\begin{figure}
    \centering
    \begin{subfigure}[b]{0.48\linewidth}
    \centering
    \def\vinescale{0.8932}%
    \input{figures/3D-maze-paths}%
    \caption{}\label{fig:3D-maze-sim}
    \end{subfigure}\hfill
    \begin{subfigure}[b]{0.48\linewidth}
    \centering
    \includegraphics[width=\linewidth]{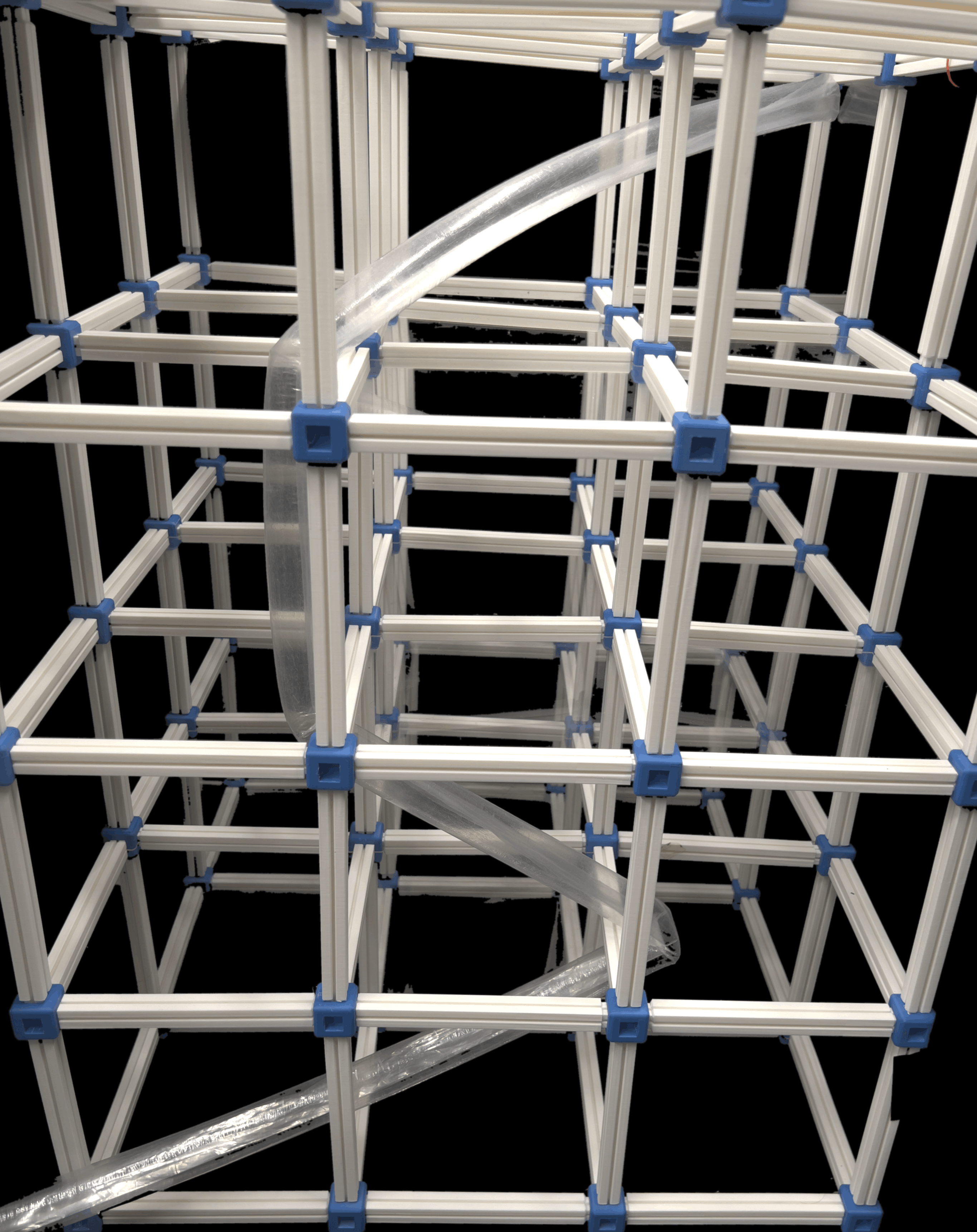}
    \caption{}\label{fig:3D-maze-hw}
    \end{subfigure}

    \vspace{1.5ex}

    \begin{subfigure}[b]{\linewidth}
    \centering
    \input{figures/3D-maze-pressure}%
    \caption{}\label{fig:3D-maze-pressure}
    \end{subfigure}
    \caption{The four minimum-pressure paths (whose predicted pressures are mutually separated by at least 1\,kPa) through a 3D obstacle course.  (a) The four paths are first determined from our model in a simulation.  (b) The minimum-pressure path being grown through hardware. (c)  Predicted growth pressure versus deployed length.  If we assume a hypothetical burst pressure of 21.5\,kPa (dashed line), only the minimum-pressure path is feasible.}
    \label{fig:3D-maze}
\end{figure}

\section{Conclusion, Limitations, and Future Work}
In this paper, we studied the minimum-pressure growth planning problem and proposed an algorithm that finds globally optimal paths in 2D and approximate solutions in 3D.  Our growth planner is a key step towards effective vine robot navigation, applying to both teleoperated and autonomous vine robots.  Due to the speed of our algorithm, it is able to plan paths for vine robots in much more complicated environments than previously possible.  Additionally, since we are able to find a minimum-pressure path (or close to it in 3D), we are able to design paths that reach much further into such environments.

We identify a few limitations that we plan to address in future work:
\begin{itemize}
    \item Like other visibility-graph planning algorithms, our algorithm requires the environment to be fully known ahead of time, while online algorithms would be preferable in some applications of vine robots.
    \item We do not quantify the exact suboptimality gap of our algorithm in 3D environments.  Future work should identify an efficient discretization strategy and quantify how quickly our method converges towards a minimum-pressure path, as is done in~\cite{papadimitrioualgorithm1985} for the minimum-distance problem.
    \item Our experiments confirm that the proposed pressure model ranks paths correctly, which is what the growth planner relies on, while also revealing some inaccuracy in the predicted absolute pressures. Future work should develop a highly reproducible experimental setup with lower variability across a wider range of paths.
    \item Our pressure model omits gravity, which is shown to be significant for large vine robots in 3D \cite{HeapInspection2025}.
    \item In our experiments, we did not test vine robots with actuators and instead manually placed the vine robot in a given path.  In either case, our model is still applicable, but future work should include actuation to actively steer through a given environment, and potentially incorporate actuator considerations \cite{GaoParallel2026} and uncertainty considerations \cite{greerrobust2020} in our planner.
\end{itemize}

\section*{Acknowledgments}
\ifanonymized
    {}
\else
    This work was supported by the Air Force Office of Scientific Research award FA9550-25-1-0375.  The views and conclusions contained in the paper are those of the authors and should not be interpreted as representing the official policies, either expressed or implied, of the United States Air Force or the U.S.\ government.  The U.S.\ government is authorized to reproduce and distribute reprints for government purposes notwithstanding any copyright notation herein.
    
    We thank Kevin Wu for designing and manufacturing the modular voxel-based obstacle course for our hardware experiments.
\fi

We disclose that Claude by Anthropic assisted us in: writing the \algoname{} Python library; generating TikZ and PGFplots code for the figures; photo editing of Fig.~\ref{fig:2D-maze-hw} and Fig.~\ref{fig:3D-maze-hw} (we emphasize that the images were not edited using any generative image techniques, and are exclusively edited with non-generative image methods such as masking and color grading); proving the lemmas in the \nameref{appendix}.

\appendix \label{appendix}
In this appendix, we provide a proof of Lemma~\ref{lem:ridges}, which is valid in 2D and 3D.
The proof is broken down into two steps.  For the first step, we only provide a proof sketch due to page constraints.

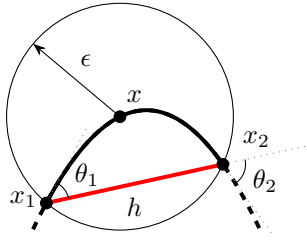
\begin{figure}
    \centering
    \input{figures/theorem-1-fig}
    \caption{Shortcutting a curved path with a straight line shown in red.}
    \label{fig:shortcut-curve}
\end{figure}

\begin{lemma*} \label{th:cutting-corner}
    If problem~\eqref{eq:planning_problem} is feasible, then it admits a globally optimal solution that is piecewise-linear.
\end{lemma*}
\begin{proof}[Proof sketch]
Assume $\gamma$ is a globally optimal solution that is not piecewise-linear.  Pick a point $x$ in the image of $\gamma$ at which $\gamma$ is twice differentiable with nonzero curvature and that is not contained within an obstacle (i.e., is in the interior of the free space).  Then there exists an $\epsilon$-ball centered at $x$ that is also disjoint from all obstacles.  Referring to Fig.~\ref{fig:shortcut-curve}, let $\tilde \gamma$ be a copy of $\gamma$, except now going through the red straight segment.  Let $\tau$ be the tension at $x_1$.  Applying \eqref{eq:tension-int-single-path}, if $\gamma(s_1)=x_1$ and $\gamma(s_2)=x_2$, the tension at $x_2$ is:
\begin{equation*}
\tau_\gamma=\tau e^{\mu\theta(s_1,s_2)}+\lambda\int_{s_1}^{s_2}e^{\mu(\theta(s_1,s_2)-\theta(s_1,t))}\,dt.
\end{equation*}

Applying \eqref{eq:pressure-recur}, we compute the tension at $x_2$ for $\tilde\gamma$:
\begin{equation*}
    \tau_{\tilde\gamma}=(\tau e^{\mu\theta_1}+\lambda h)e^{\mu\theta_2}.
\end{equation*}

We aim to show $\tau_\gamma\ge\tau_{\tilde\gamma}$.  It suffices to show the following two inequalities:
\begin{equation}
    \tau(e^{\mu\theta(s_1,s_2)}-e^{\mu(\theta_1+\theta_2)}) \ge0\label{eq:term1}
\end{equation}
and
\begin{equation}
    \lambda\left(\int_{s_1}^{s_2}e^{\mu\big(\theta(s_1,s_2)-\theta(s_1,t)\big)}\,dt-he^{\mu\theta_2}\right)\ge0. \label{eq:term2}
\end{equation}

From~\cite{milnortotal1950} we have that $\theta(s_1,s_2)\ge\theta_1+\theta_2$, so \eqref{eq:term1} holds.

To show \eqref{eq:term2}, we rewrite it as a function of the arclength, allowing us to slide the basepoint of the straight segment from $x_1$ to $\gamma(s)$.  We call this function $g$:
\begin{equation*}
    g(s)=\int_s^{s_2}e^{\mu(\theta(s,s_2)-\theta(s,t))}\,dt-h(s)e^{\mu\theta_2(s)},
\end{equation*}
where $s\in[s_1,s_2]$, $h(s)$ computes the length of the straight segment, and $\theta_2(s)$ computes the new angle at $x_2$.  Showing $g(s_1)\ge0$ is equivalent to showing \eqref{eq:term2}.  Notice that $g(s_2)=0$.  We do not prove it here for reasons of space, but it can be verified that $g'(s)\le0$ when $s\in[s_1,s_2]$.  This implies $g(s)\ge0$ when $s\in[s_1,s_2]$ and thus shows $\tau_\gamma\ge\tau_{\tilde\gamma}$. 

Let $P_\gamma$, $P_{\tilde \gamma}$ be the pressure functions and $L_\gamma$, $L_{\tilde \gamma}$ be lengths of the paths $\gamma$ and $\tilde\gamma$.  Since the paths coincide after $x_2$ and, by \eqref{eq:pressure-differential}, endpoint tension is nondecreasing in the tension at $x_2$, $\tau_\gamma\ge\tau_{\tilde\gamma}$ implies $P_{\gamma}(L_\gamma)\ge P_{\tilde\gamma}(L_{\tilde\gamma})$.  Since we assumed that $\gamma$ is a globally optimal solution, we must have $P_{\gamma}(L_\gamma)\le P_{\tilde\gamma}(L_{\tilde\gamma})$.  Therefore, we conclude that $P_{\gamma}(L_\gamma)= P_{\tilde\gamma}(L_{\tilde\gamma})$ and $\tilde\gamma$ is also a globally optimal solution.

By iterating this replacement we can construct a globally optimal solution that is piecewise-linear, thus completing the proof.
\end{proof}

Previous work has often assumed curvature in vine robots to be beneficial~\cite{selvaggioobstacle2020} and while in some applications this can be the case, this proof shows that curving, rather than bending, never reduces growth pressure.

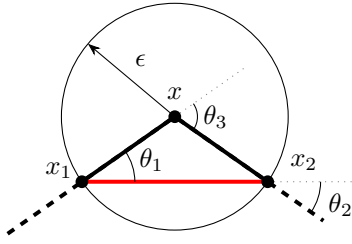
\begin{figure}
    \centering
    \input{figures/theorem-2-fig}
    \caption{Shortcutting a piecewise-linear path with a straight line shown in red.}
    \label{fig:shortcut-linear}
\end{figure}

\begin{lemma*}
    If problem~\eqref{eq:planning_problem} is feasible, then it admits a globally optimal solution that is piecewise-linear with breakpoints exclusively on obstacle ridges.
\end{lemma*}

\begin{proof}
    By the previous lemma, there exists at least one globally optimal solution that is piecewise-linear.  Among these solutions, let $\gamma$ be one of minimum length.  Let $x$ be a breakpoint of $\gamma$ that does not lie on an obstacle ridge.  If $x$ does not lie on any obstacles, there exists an $\epsilon$-ball centered at $x$ that is also disjoint from all obstacles.  If $x$ lies on an obstacle facet, there exists an $\epsilon$-ball centered at $x$ that is disjoint from all but one obstacle.  Half of the $\epsilon$-ball is disjoint from the obstacle.  In either case, the disjoint portion of the $\epsilon$-ball is convex so the red straight segment does not enter any obstacle.  Let $x_1,x_2$ be points on the two segments adjacent to $x$ within the ball, with angles as depicted in Fig.~\ref{fig:shortcut-linear}.  Let $\tilde \gamma$ be a copy of $\gamma$, except now going through the red straight segment.  We aim to show $\tau_\gamma\ge\tau_{\tilde\gamma}$, where $\tau_\gamma,\tau_{\tilde\gamma}$ are computed using \eqref{eq:pressure-recur}.  Let $g$ be:
    \begin{equation*}
        g(\theta)=\frac{\sin\theta}{e^{\mu\theta}-\cos\theta},
    \end{equation*}
    where $\theta\in(0,\pi]$.  Since $\theta_3=\theta_1+\theta_2$, the tension terms cancel, and using the law of sines, we can rewrite the inequality as $g(\theta_2)\ge g(\theta_3)$.  Computing $g'$ we obtain:
    \begin{equation*}
        g'(\theta)=\frac{e^{\mu\theta}(\cos\theta-\mu\sin\theta)-1}{(\cos\theta-e^{\mu\theta})^2}.
    \end{equation*}
    The derivative of the numerator of $g'$ is $-(1+\mu^2)e^{\mu\theta}\sin\theta$, which is nonpositive on $[0,\pi]$.  The numerator of $g'$ evaluated at zero equals zero, so we have $g'\le0$ on $(0,\pi]$. Since $\theta_2\le\theta_3\le\pi$, we have that $g(\theta_2)\ge g(\theta_3)$, which shows $\tau_\gamma\ge\tau_{\tilde\gamma}$.  Therefore, by the same logic in the proof of the previous lemma, we conclude $P_{\gamma}(L_\gamma)= P_{\tilde\gamma}(L_{\tilde\gamma})$.

    The solution $\tilde\gamma$ is piecewise-linear and, by the triangle inequality, strictly shorter than $\gamma$.  This is a contradiction of our initial assumption that $\gamma$ is a globally optimal solution of minimum length.  Therefore, $\gamma$ must have breakpoints exclusively on obstacle ridges, thus concluding the proof.
\end{proof}

\bibliographystyle{IEEEtran}
\bibliography{references}

\vspace{12pt}

\end{document}

%% file: preambles/2d-maze-preamble.tex
\gdef\vineplannerstyles{}

\usetikzlibrary{shapes.geometric, calc}
\pgfdeclarelayer{vineobstacles}
\pgfsetlayers{vineobstacles,main}

\definecolor{vineobstacle}{HTML}{1F77B4}
\definecolor{vinedistance}{HTML}{D62728}
\definecolor{vineangle}{HTML}{FF7F0E}
\definecolor{vinepressure}{HTML}{FF00FF}
\definecolor{vinestart}{HTML}{008000}
\definecolor{vinetarget}{HTML}{800080}

\pgfplotsset{
  vine/course axis/.style={
    axis equal image, axis on top=false,
    tick align=outside, tick pos=left,
    grid=both, grid style={draw=black!15, line width=0.3pt},
    legend cell align=left,
    legend style={font=\scriptsize, fill=white, fill opacity=1,
                  text opacity=1, draw=black!60, rounded corners=1pt},
  },
  vine/pressure axis/.style={
    axis on top=false,
    tick align=outside, tick pos=left,
    grid=both, grid style={draw=black!15, line width=0.3pt},
    legend cell align=left,
    legend style={font=\scriptsize, fill=white, fill opacity=1,
                  text opacity=1, draw=black!60, rounded corners=1pt},
  },
  vine/obstacle/.style={draw=none, fill=vineobstacle, fill opacity=0.4},
  vine/vertex/.style={
    only marks, mark=*,
    mark options={fill=vineobstacle, draw=vineobstacle, scale=0.35},
  },
  vine/path/.style={line width=0.5pt, line join=round, line cap=round},
  vine/distance path/.style={vine/path, draw=vinedistance, solid},
  vine/angle path/.style={vine/path, draw=vineangle, dashed},
  vine/pressure path/.style={vine/path, draw=vinepressure, solid},
  vine/burst/.style={draw=black, line width=0.9pt, densely dashed},
}

\tikzset{
  vine/obstacle label/.style={font=\tiny\bfseries, inner sep=0pt},
  vine/start/.style={
    circle, minimum size=3pt, inner sep=0pt, draw=none, fill=vinestart,
  },
  vine/target/.style={
    star, star points=5, star point ratio=2.3,
    minimum size=5pt, inner sep=0pt, draw=none, fill=vinetarget,
  },
}

\pgfdeclareplotmark{vinestart}{\node[vine/start]{};}
\pgfdeclareplotmark{vinetarget}{\node[vine/target]{};}

%% file: preambles/scaling-preamble.tex
\definecolor{vinetotal}{HTML}{000000}
\definecolor{vinegraph}{HTML}{1F77B4}
\definecolor{vinesearch}{HTML}{D62728}
\pgfplotsset{
  vine/scaling axis/.style={
    tick align=outside, tick pos=left,
    grid=both,
    major grid style={draw=black!18, line width=0.3pt},
    minor grid style={draw=black!8,  line width=0.2pt},
    every axis plot/.append style={line width=0.9pt, mark size=1.9pt},
    legend cell align=left,
    legend style={font=\footnotesize, fill=white, fill opacity=1,
                  text opacity=1, draw=black!60, rounded corners=1pt},
    label style={font=\small},
    tick label style={font=\footnotesize},
  },
  vine/scaling total/.style={
    color=vinetotal, mark=*, mark options={fill=vinetotal, draw=vinetotal},
  },
  vine/scaling visibility/.style={
    color=vinegraph, mark=square*, mark options={fill=vinegraph, draw=vinegraph},
  },
  vine/scaling pressure/.style={
    color=vinesearch, mark=triangle*,
    mark options={fill=vinesearch, draw=vinesearch},
  },
}
\tikzset{
  vine/scaling callout/.style={
    draw=black!55, dashed, dash pattern=on 2.4pt off 1.8pt, line width=0.6pt,
  },
  vine/scaling callout label/.style={
    font=\footnotesize, text=black!70, anchor=west, xshift=2.5pt,
    fill=white, inner sep=1.6pt, rounded corners=0.6pt,
  },
}

%% file: figures/2D-maze-paths.tex
% colours matched to the hardware photograph
\definecolor{pathoptimal}{RGB}{82,217,94}
\definecolor{paththird}{RGB}{224,138,63}
\definecolor{pathfourth}{RGB}{82,138,217}
\definecolor{pathfifth}{RGB}{219,70,75}
% \definecolor{vineobstacle}{HTML}{1F77B4}
\definecolor{vineobstacle}{RGB}{0,0,0}
%% Four routes through the simulated maze, ranked by growth pressure.
%% Geometry from vine_planner (PathPlanner.compute_pressure_path), in metres.
%% Rotated 180 deg from the notebook view so the start is bottom-left.
\begin{tikzpicture}
\begin{axis}[
    scale only axis=true, enlargelimits=false,
    width=\linewidth, height=0.84098\linewidth,
    xmin=-0.02200, xmax=0.93640, ymin=-0.02200, ymax=0.78400,
    hide axis, clip=false,
]
\draw[vineobstacle, line width=3.0pt, line cap=butt] (axis cs:0.6096,0.1524) -- (axis cs:0.7620,0.1524);
\draw[vineobstacle, line width=3.0pt, line cap=butt] (axis cs:0.3048,0.1524) -- (axis cs:0.4572,0.1524);
\draw[vineobstacle, line width=3.0pt, line cap=butt] (axis cs:0.1524,0.1524) -- (axis cs:0.3048,0.1524);
\draw[vineobstacle, line width=3.0pt, line cap=butt] (axis cs:0.0000,0.1524) -- (axis cs:0.1524,0.1524);
\draw[vineobstacle, line width=3.0pt, line cap=butt] (axis cs:0.6096,0.3048) -- (axis cs:0.7620,0.3048);
\draw[vineobstacle, line width=3.0pt, line cap=butt] (axis cs:0.4572,0.3048) -- (axis cs:0.6096,0.3048);
\draw[vineobstacle, line width=3.0pt, line cap=butt] (axis cs:0.1524,0.3048) -- (axis cs:0.3048,0.3048);
\draw[vineobstacle, line width=3.0pt, line cap=butt] (axis cs:0.7620,0.4572) -- (axis cs:0.9144,0.4572);
\draw[vineobstacle, line width=3.0pt, line cap=butt] (axis cs:0.4572,0.4572) -- (axis cs:0.6096,0.4572);
\draw[vineobstacle, line width=3.0pt, line cap=butt] (axis cs:0.3048,0.4572) -- (axis cs:0.4572,0.4572);
\draw[vineobstacle, line width=3.0pt, line cap=butt] (axis cs:0.1524,0.4572) -- (axis cs:0.3048,0.4572);
\draw[vineobstacle, line width=3.0pt, line cap=butt] (axis cs:0.7620,0.6096) -- (axis cs:0.9144,0.6096);
\draw[vineobstacle, line width=3.0pt, line cap=butt] (axis cs:0.6096,0.6096) -- (axis cs:0.7620,0.6096);
\draw[vineobstacle, line width=3.0pt, line cap=butt] (axis cs:0.4572,0.6096) -- (axis cs:0.6096,0.6096);
\draw[vineobstacle, line width=3.0pt, line cap=butt] (axis cs:0.1524,0.6096) -- (axis cs:0.3048,0.6096);
%% Border: same cell segments, colour and weight as the obstacles.
%% Top and bottom run the full width; the openings are in the side walls,
%% at the cell the vine enters (bottom left) and leaves by (top right).
\draw[vineobstacle, line width=3.0pt, line cap=rect] (axis cs:0.9144,0.0000) -- (axis cs:0.9144,0.1524);
\draw[vineobstacle, line width=3.0pt, line cap=rect] (axis cs:0.0000,0.1524) -- (axis cs:0.0000,0.3048);
\draw[vineobstacle, line width=3.0pt, line cap=rect] (axis cs:0.9144,0.1524) -- (axis cs:0.9144,0.3048);
\draw[vineobstacle, line width=3.0pt, line cap=rect] (axis cs:0.0000,0.3048) -- (axis cs:0.0000,0.4572);
\draw[vineobstacle, line width=3.0pt, line cap=rect] (axis cs:0.9144,0.3048) -- (axis cs:0.9144,0.4572);
\draw[vineobstacle, line width=3.0pt, line cap=rect] (axis cs:0.0000,0.4572) -- (axis cs:0.0000,0.6096);
\draw[vineobstacle, line width=3.0pt, line cap=rect] (axis cs:0.9144,0.4572) -- (axis cs:0.9144,0.6096);
\draw[vineobstacle, line width=3.0pt, line cap=rect] (axis cs:0.0000,0.6096) -- (axis cs:0.0000,0.7620);
\draw[vineobstacle, line width=3.0pt, line cap=butt] (axis cs:0.0000,0.0000) -- (axis cs:0.1524,0.0000);
\draw[vineobstacle, line width=3.0pt, line cap=butt] (axis cs:0.0000,0.7620) -- (axis cs:0.1524,0.7620);
\draw[vineobstacle, line width=3.0pt, line cap=butt] (axis cs:0.1524,0.0000) -- (axis cs:0.3048,0.0000);
\draw[vineobstacle, line width=3.0pt, line cap=butt] (axis cs:0.1524,0.7620) -- (axis cs:0.3048,0.7620);
\draw[vineobstacle, line width=3.0pt, line cap=butt] (axis cs:0.3048,0.0000) -- (axis cs:0.4572,0.0000);
\draw[vineobstacle, line width=3.0pt, line cap=butt] (axis cs:0.3048,0.7620) -- (axis cs:0.4572,0.7620);
\draw[vineobstacle, line width=3.0pt, line cap=butt] (axis cs:0.4572,0.0000) -- (axis cs:0.6096,0.0000);
\draw[vineobstacle, line width=3.0pt, line cap=butt] (axis cs:0.4572,0.7620) -- (axis cs:0.6096,0.7620);
\draw[vineobstacle, line width=3.0pt, line cap=butt] (axis cs:0.6096,0.0000) -- (axis cs:0.7620,0.0000);
\draw[vineobstacle, line width=3.0pt, line cap=butt] (axis cs:0.6096,0.7620) -- (axis cs:0.7620,0.7620);
\draw[vineobstacle, line width=3.0pt, line cap=butt] (axis cs:0.7620,0.0000) -- (axis cs:0.9144,0.0000);
\draw[vineobstacle, line width=3.0pt, line cap=butt] (axis cs:0.7620,0.7620) -- (axis cs:0.9144,0.7620);
\addplot[pathfifth, line width=1.8pt, line join=round, line cap=round, mark=none] coordinates {(0.00000,0.00000) (0.45872,0.15240) (0.45872,0.15392) (0.15088,0.30480) (0.15088,0.30632) (0.15088,0.45872) (0.45568,0.61112) (0.91440,0.76200)};
\addplot[pathfourth, line width=1.8pt, line join=round, line cap=round, mark=none] coordinates {(0.00000,0.00000) (0.76352,0.15240) (0.76352,0.30632) (0.61112,0.45872) (0.45568,0.60960) (0.45568,0.61112) (0.91440,0.76200)};
\addplot[paththird, line width=1.8pt, line join=round, line cap=round, mark=none] coordinates {(0.00000,0.00000) (0.45872,0.15240) (0.45872,0.15392) (0.30632,0.30632) (0.15088,0.45720) (0.15088,0.45872) (0.15088,0.61112) (0.91440,0.76200)};
\addplot[pathoptimal, line width=1.8pt, line join=round, line cap=round, mark=none] coordinates {(0.00000,0.00000) (0.45872,0.15240) (0.45872,0.15392) (0.30632,0.30632) (0.15088,0.45720) (0.15088,0.45872) (0.45568,0.61112) (0.91440,0.76200)};
% \addplot[only marks, mark=*, mark size=2.2pt, black] coordinates {(0.0000,0.0000)};
% \addplot[only marks, mark=square*, mark size=2.2pt, black] coordinates {(0.9144,0.7620)};
\end{axis}
\end{tikzpicture}

%% file: figures/2D-maze-pressure.tex
% colours matched to the hardware photograph
\definecolor{pathbest}{RGB}{82,217,94}
\definecolor{path2nd}{RGB}{82,138,217}
\definecolor{path3rd}{RGB}{224,138,63}
\definecolor{path4th}{RGB}{219,70,75}
%% Growth pressure against deployed length, in kPa and metres.
%% The legend sits outside the axis, to its right.
\begin{tikzpicture}
\begin{axis}[
    width=0.80\linewidth, height=0.54\linewidth,
    xlabel={Deployed length (m)}, ylabel={Predicted pressure (kPa)},
    xmin=0, xmax=1.886, ymin=6.5, ymax=31.1,
    xtick={0,0.5,1.0,1.5},
    ytick={10,15,20,25,30},
    tick align=outside, tick pos=left,
    tick label style={font=\scriptsize}, label style={font=\small},
    grid=both, grid style={draw=black!15, line width=0.3pt},
    reverse legend,
    legend cell align=left,
    legend style={font=\small, at={(1.04,1.0)}, anchor=north west,
                  draw=none, fill=none, row sep=2pt,
                  /tikz/every even column/.append style={column sep=3pt}},
]
\addplot[path4th, line width=1.8pt, mark=none] coordinates {(0.00000,7.3429) (0.48338,7.7158) (0.48338,10.5141) (0.48490,10.5153) (0.48490,14.1340) (0.82773,14.3984) (0.82773,19.6514) (0.82926,19.6526) (0.82926,19.6526) (0.98166,19.7701) (0.98166,27.2203) (1.32243,27.4831) (1.32243,28.6855) (1.80533,29.0579)};
\addlegendentry{4th}
\addplot[path3rd, line width=1.8pt, mark=none] coordinates {(0.00000,7.3429) (0.77859,7.9435) (0.77859,11.2068) (0.93251,11.3256) (0.93251,13.9666) (1.14804,14.1328) (1.14804,14.1904) (1.36466,14.3575) (1.36466,17.9246) (1.36619,17.9257) (1.36619,25.6733) (1.84909,26.0458)};
\addlegendentry{3rd}
\addplot[path2nd, line width=1.8pt, mark=none] coordinates {(0.00000,7.3429) (0.48338,7.7158) (0.48338,10.5141) (0.48490,10.5153) (0.48490,12.9288) (0.70043,13.0951) (0.70043,13.1477) (0.91706,13.3148) (0.91706,16.5830) (0.91858,16.5841) (0.91858,16.5841) (1.07098,16.7017) (1.07098,24.7212) (1.84927,25.3216)};
\addlegendentry{2nd}
\addplot[pathbest, line width=1.8pt, mark=none] coordinates {(0.00000,7.3429) (0.48338,7.7158) (0.48338,10.5141) (0.48490,10.5153) (0.48490,12.9288) (0.70043,13.0951) (0.70043,13.1477) (0.91706,13.3148) (0.91706,16.5830) (0.91858,16.5841) (0.91858,22.7048) (1.25936,22.9677) (1.25936,23.9576) (1.74225,24.3301)};
\addlegendentry{Best}
\addplot[black!55, dashed, line width=1.0pt, mark=none, forget plot] coordinates {(0,25.0) (1.886,25.0)};
\node[anchor=south west, font=\scriptsize, black!65, inner sep=2pt] at (axis cs:0.03,25.0) {burst pressure};
\end{axis}
\end{tikzpicture}

%% file: figures/simple-vine.tex
\begin{subfigure}[t]{0.60\linewidth}
    \centering
    % ---------- vine robot with two turns (piecewise linear) -----------
    \pgfmathsetlength{\vineunit}{\linewidth/9.32}
    \begin{tikzpicture}[x=\vineunit,y=\vineunit,>=Latex,font=\footnotesize]
        \pgfmathsetmacro{\lA}{3.0}\pgfmathsetmacro{\lB}{3.4}\pgfmathsetmacro{\lC}{2.6}
        \pgfmathsetmacro{\hA}{30}\pgfmathsetmacro{\hB}{-16}\pgfmathsetmacro{\hC}{44}
        \pgfmathsetmacro{\Wb}{0.42}\pgfmathsetmacro{\hw}{\Wb/2}
        \pgfmathsetmacro{\Rarc}{1.20}\pgfmathsetmacro{\Rlab}{1.70}\pgfmathsetmacro{\bias}{0.30}
        \pgfmathsetmacro{\Cext}{1.80}
        \pgfmathsetmacro{\doffA}{0.95}\pgfmathsetmacro{\dlabA}{1.32}
        \pgfmathsetmacro{\doffB}{1.55}\pgfmathsetmacro{\dlabB}{1.92}
        \pgfmathsetlength{\bodyouter}{\Wb\vineunit}
        \pgfmathsetlength{\bodyinner}{\dimexpr\bodyouter-1.2pt\relax}
        \tikzset{ext/.style={black!45,line width=0.3pt},
               dim/.style={black!70,line width=0.4pt,{Bar[width=4pt]}-{Bar[width=4pt]}},
               tag/.style={circle,draw=black,fill=white,inner sep=1pt,minimum size=11pt,font=\scriptsize}}
        \coordinate (P0) at (0,0);
        \coordinate (P1) at ($(P0)+(\hA:\lA)$);
        \coordinate (P2) at ($(P1)+(\hB:\lB)$);
        \coordinate (P3) at ($(P2)+(\hC:\lC)$);
        % inner-wall miter vertices: turn 1 rides the lower wall, turn 2 the upper wall
        \pgfmathsetmacro{\bisA}{(\hA+\hB)/2-90}\pgfmathsetmacro{\radA}{\hw/cos((\hA-\hB)/2)}
        \pgfmathsetmacro{\bisB}{(\hB+\hC)/2+90}\pgfmathsetmacro{\radB}{\hw/cos((\hC-\hB)/2)}
        \coordinate (V1) at ($(P1)+(\bisA:\radA)$);
        \coordinate (V2) at ($(P2)+(\bisB:\radB)$);
        % headings of the three RED segments, normalised to (-180,180]
        \pgfmathanglebetweenpoints{\pgfpointanchor{P0}{center}}{\pgfpointanchor{V1}{center}}%
        \edef\tmp{\pgfmathresult}\pgfmathsetmacro{\sA}{\tmp>180?\tmp-360:\tmp}
        \pgfmathanglebetweenpoints{\pgfpointanchor{V1}{center}}{\pgfpointanchor{V2}{center}}%
        \edef\tmp{\pgfmathresult}\pgfmathsetmacro{\sB}{\tmp>180?\tmp-360:\tmp}
        \pgfmathanglebetweenpoints{\pgfpointanchor{V2}{center}}{\pgfpointanchor{P3}{center}}%
        \edef\tmp{\pgfmathresult}\pgfmathsetmacro{\sC}{\tmp>180?\tmp-360:\tmp}
        \useasboundingbox (-0.50,-1.15) rectangle (8.82,3.05);  % 9.32 x 4.20
        % body
        \draw[line width=\bodyouter,line join=miter,black] (P0)--(P1)--(P2)--(P3);
        \draw[line width=\bodyinner,line join=miter,white] (P0)--(P1)--(P2)--(P3);
        % deviation angles between consecutive red segments, at the red vertices
        \draw[black!45,dashed,line width=0.4pt] (V1)--($(V1)+(\sA:\Cext)$);
        \draw[black!70,line width=0.5pt] (V1)+(\sA:\Rarc) arc[start angle=\sA,end angle=\sB,radius=\Rarc];
        \node at ($(V1)+({\sA+\bias*(\sB-\sA)}:\Rlab)$) {$\theta_1$};
        \draw[black!45,dashed,line width=0.4pt] (V2)--($(V2)+(\sB:\Cext)$);
        \draw[black!70,line width=0.5pt] (V2)+(\sB:\Rarc) arc[start angle=\sB,end angle=\sC,radius=\Rarc];
        \node at ($(V2)+({\sB+\bias*(\sC-\sB)}:\Rlab)$) {$\theta_2$};
        % l1 : red segment P0--V1, above
        \draw[ext] (P0)--($(P0)+({\sA+90}:\doffA)$);
        \draw[ext] (V1)--($(V1)+({\sA+90}:\doffA)$);
        \draw[dim] ($(P0)+({\sA+90}:\doffA)$)--($(V1)+({\sA+90}:\doffA)$);
        \node at ($(P0)!0.5!(V1)+({\sA+90}:\dlabA)$) {$ l_1$};
        % l2 : red segment V1--V2, below
        \draw[ext] (V1)--($(V1)+({\sB-90}:\doffB)$);
        \draw[ext] (V2)--($(V2)+({\sB-90}:\doffB)$);
        \draw[dim] ($(V1)+({\sB-90}:\doffB)$)--($(V2)+({\sB-90}:\doffB)$);
        \node at ($(V1)!0.5!(V2)+({\sB-90}:\dlabB)$) {$ l_2$};
        % l3 : red segment V2--P3, below
        \draw[ext] (V2)--($(V2)+({\sC-90}:\doffB)$);
        \draw[ext] (P3)--($(P3)+({\sC-90}:\doffB)$);
        \draw[dim] ($(V2)+({\sC-90}:\doffB)$)--($(P3)+({\sC-90}:\doffB)$);
        \node at ($(V2)!0.5!(P3)+({\sC-90}:\dlabB)$) {$ l_3$};
        % path: turns ride the inner wall
        \draw[line width=1.5pt,line join=miter,line cap=round,red!65!black] (P0)--(V1)--(V2)--(P3);
        \foreach \p in {P0,V1,V2,P3}{\fill[red!45!black] (\p) circle (1.5pt);}
        % circled labels
        \node[tag] at ($(P0)+(-62:0.62)$)   {0};
        \node[tag] at ($(V1)+(-62:0.80)$)   {1};
        \node[tag] at ($(V2)+(\bisB:0.62)$) {2};
        \node[tag] at ($(P3)+(50:0.55)$)    {3};
    \end{tikzpicture}
    % -------------------------------------------------------------------
    \caption{}
    % \caption{Simple vine robot with three straight segments $ l_1,  l_2, l_3$ and two turns $\theta_1,\theta_2$.  We assume the initial tension is $T$ which is at the start of the robot.}
    \label{fig:simple-vine}
\end{subfigure}

%% file: figures/two-cylinders.tex
\hfill
\begin{subfigure}[t]{0.36\linewidth}
    \centering
    % ---------- capstan schematic: rope over two cylinders -------------
    % Parameters: radius, three segment headings [deg], three lengths.
    % Wrap angles follow: theta_1 = phiA - phiB, theta_2 = phiC - phiB.
    \pgfmathsetlength{\capstanunit}{\linewidth/7.5}
    \begin{tikzpicture}[x=\capstanunit,y=\capstanunit,>=Latex,font=\footnotesize]
      \pgfmathsetmacro{\R}{1.0}
      \pgfmathsetmacro{\phiA}{30}     % held segment heading
      \pgfmathsetmacro{\phiB}{-78}    % inter-cylinder segment heading
      \pgfmathsetmacro{\phiC}{32}     % outgoing segment heading
      \pgfmathsetmacro{\Lmid}{3.0}
      \pgfmathsetmacro{\Lin}{1.9}
      \pgfmathsetmacro{\Lout}{1.9}
      % derived tangency angles
      \pgfmathsetmacro{\aIn}{\phiA+90}
      \pgfmathsetmacro{\aMidI}{\phiB+90}
      \pgfmathsetmacro{\aMidII}{\phiB-90}
      \pgfmathsetmacro{\aOut}{\phiC-90}
      % derived points
      \coordinate (C1) at (0,0);
      \coordinate (T0) at ($(C1)+(\aIn:\R)$);
      \coordinate (T1) at ($(C1)+(\aMidI:\R)$);
      \coordinate (T2) at ($(T1)+(\phiB:\Lmid)$);
      \coordinate (C2) at ($(T2)+(\aMidI:\R)$);
      \coordinate (T3) at ($(C2)+(\aOut:\R)$);
      \coordinate (P0) at ($(T0)-(\phiA:\Lin)$);
      \coordinate (P3) at ($(T3)+(\phiC:\Lout)$);
      \useasboundingbox (-2.20,-3.60) rectangle (5.10,1.15);
      % cylinders
      \draw[fill=black!8,draw=black,line width=0.5pt] (C1) circle (\R);
      \draw[fill=black!8,draw=black,line width=0.5pt] (C2) circle (\R);
      \fill (C1) circle (1pt);  \fill (C2) circle (1pt);
      % wrap angles
      \draw[black!55,line width=0.4pt] (C1) -- (T0);
      \draw[black!55,line width=0.4pt] (C1) -- (T1);
      \draw[black!70,line width=0.5pt]
            (C1)+(\aMidI:0.3) arc[start angle=\aMidI,end angle=\aIn,radius=0.3];
      \node at ($(C1)+({(\aIn+\aMidI)/2}:0.65)$) {$\theta_1$};
      \draw[black!55,line width=0.4pt] (C2) -- (T2);
      \draw[black!55,line width=0.4pt] (C2) -- (T3);
      \draw[black!70,line width=0.5pt]
            (C2)+(\aMidII:0.3) arc[start angle=\aMidII,end angle=\aOut,radius=0.3];
      \node at ($(C2)+({(\aMidII+\aOut)/2}:0.6)$) {$\theta_2$};
      % rope
      \draw[line width=2.2pt,line cap=round,red!65!black,->]
            (P0) -- (T0) arc[start angle=\aIn,    end angle=\aMidI, radius=\R]
            -- (T2) arc[start angle=\aMidII, end angle=\aOut,  radius=\R] -- (P3);
      % tensions
      \node at ($(P0)!0.5!(T0)+({\phiA+90}:0.8)$) {$T_{\mathrm{hold}}$};
      \node at ($(T1)!0.5!(T2)+({\phiB+90}:1.15)$) {$T_{\mathrm{out},1}$};
      \node at ($(P3)+({\phiC+90}:0.50)$)          {$T_{\mathrm{out},2}$};
    \end{tikzpicture}
    % -------------------------------------------------------------------
    % \captionsetup{justification=raggedright, singlelinecheck=false}
    % \caption{Applying the Capstan equation to two consecutive cylinders, we obtain $T_{\mathrm{out},2}=T_{\mathrm{hold}}e^{\theta_1}e^{\theta_2}$ and thus contributes multiplicatively.}
    \caption{}
    \label{fig:capstan}
\end{subfigure}%

%% file: figures/runtime-scaling.tex
% Generated by make_scaling_tikz.py from scaling_summary.csv -- do not
% edit by hand; re-run the generator instead.
%
% Body only: no figure environment, no caption, no label. Drop into an
% existing figure with \input{runtime-scaling}. Needs
% \input{scaling-preamble} in the document preamble.
%
% xmode/ymode must stay in the axis options: pgfplots fixes the axis mode
% before styles expand, so they cannot be moved into vine/scaling axis.
  \begin{tikzpicture}
    \begin{axis}[
      name=vinescaling,
      xmode=log, ymode=log,
      log basis x=10, log basis y=10,
      vine/scaling axis,
      scale only axis,
      width=0.78\columnwidth,
      height=0.60\columnwidth,
      xlabel={Vertices in the visibility graph, $n$},
      ylabel={Wall-clock time (s)},
      xmin=280, xmax=1.3e+06,
      ymin=0.0018, ymax=22000,
      xtick={1e3,1e4,1e5,1e6},
      ytick={1e-2,1e-1,1e0,1e1,1e2,1e3,1e4},
      legend pos=north west,
    ]
      % callout rule first, so the data markers draw on top of it
      \draw[vine/scaling callout]
        (axis cs:54712,0.0018) --
        (axis cs:54712,22000)
        node[vine/scaling callout label, pos=0.049]
        {Fig.~\ref{fig:dense-maze}};
      \addplot[vine/scaling visibility]
        coordinates {(403,0.003469) (1236,0.0183188) (4072,0.159768) (12263,1.3115) (40986,11.1198) (54712,20.5917) (81908,42.8099) (122890,95.7045) (245651,387.743) (409138,1101.93) (614604,2488.81) (819400,5087.59)};
      \addlegendentry{visibility graph ($\propto n^{2.03}$)}
      \addplot[vine/scaling pressure]
        coordinates {(403,0.0130482) (1236,0.0313212) (4072,0.0762987) (12263,0.157494) (40986,0.326495) (54712,0.415639) (81908,0.487476) (122890,0.600368) (245651,0.863212) (409138,1.21312) (614604,1.40973) (819400,1.73643)};
      \addlegendentry{Dijkstra's algorithm ($\propto n^{0.55}$)}
      \addplot[vine/scaling total]
        coordinates {(403,0.0166774) (1236,0.0497772) (4072,0.235205) (12263,1.4559) (40986,11.4406) (54712,21.022) (81908,43.2865) (122890,96.3163) (245651,388.606) (409138,1103.17) (614604,2490.22) (819400,5089.33)};
      \addlegendentry{total ($\propto n^{2.02}$)}
    \end{axis}
    % Second x axis along the top of the same box, y axis hidden.
    % Ticks sit at the measured $n$ of the round-obstacle-count
    % courses, so no constant vertices-per-obstacle ratio is assumed.
    \begin{axis}[
      xmode=log, log basis x=10,
      scale only axis,
      width=0.78\columnwidth,
      height=0.60\columnwidth,
      at={(vinescaling.south west)}, anchor=south west,
      axis x line*=top,
      hide y axis,
      tick align=outside,
      xmin=280, xmax=1.3e+06,
      ymin=0, ymax=1,
      xtick={403,4072,40986,409138},
      xticklabels={$10^{2}$,$10^{3}$,$10^{4}$,$10^{5}$},
      xlabel={Obstacles, $m$},
      label style={font=\small},
      tick label style={font=\footnotesize},
    ]
    \end{axis}
  \end{tikzpicture}

%% file: figures/dense-maze-spy.tex
% =====================================================================
% Magnified inset ("spy") over dense-maze.tex
% ---------------------------------------------------------------------
% dense-maze.tex is included unmodified. Everything specific to the inset
% lives here, so regenerating the figure from Python cannot clobber it.
%
% Requires, in the document preamble:
%   \usepackage{pgfplots}
%   \pgfplotsset{compat=1.18}
%   \usetikzlibrary{spy, calc}
% =====================================================================

% ---- what to magnify, in course units (metres here) ------------------
% The two paths leave the start together, separate briefly, then run as one
% from (2.3915, 1.9013) through fourteen shared vertices to (3.0174, 2.9121),
% which is the last point they have in common -- they meet again only at the
% target. Either end of that run is worth a look:
%
%   (3.017, 2.912)  the divergence point, where the paths split for good
%   (2.667, 2.454)  the arc-length midpoint of the shared run, where the
%                   two lie on top of one another
%
\def\SpyX{3.017}
\def\SpyY{2.912}

% ---- the axis limits of dense-maze.tex, needed to place the target ---
% These must match `xmin/xmax` and `ymin/ymax` in the generated file.
\def\CourseXmin{0}\def\CourseXmax{8}
\def\CourseYmin{0}\def\CourseYmax{8}

% ---- inset appearance ------------------------------------------------
\def\SpyMagnification{6}   % zoom factor
\def\SpySize{2cm}          % diameter of the magnified circle
\def\SpyColour{green!55!black}
\def\SpyLine{1pt}          % both circles and the connector between them
% Where the magnified circle sits, as a fraction of the axis box.
\def\SpyNodeU{0.8}
\def\SpyNodeV{0.2}

% ---- figure size -----------------------------------------------------
% The generated axis asks for width=\linewidth but says nothing about
% height, so pgfplots falls back to its default (about 207pt). Under
% `axis equal image` the smaller of the two bounds wins, and for a square
% course that is the height -- which is why the figure came out at 85% of
% the column with empty margins either side. Allowing a full \linewidth of
% height lets the width bind instead, taking the figure to 99%.
%
% For a course that is taller than it is wide, raise this past \linewidth.
\def\CourseHeight{\linewidth}

% ---- fonts -----------------------------------------------------------
% The generated `vine/course axis` sets a font for the legend and for
% nothing else, so the tick labels and the axis labels fall through to the
% document's normal size -- 10pt against the legend's 7pt. That is the
% mismatch: the legend was never the odd one out, the axis was.
%
% These match the hand-built figures (ticks \scriptsize, labels \small).
\def\CourseTickFont{\scriptsize}
\def\CourseLabelFont{\small}

% ---- legend ----------------------------------------------------------
% Position of the legend's top-left corner, as a fraction of the axis box.
\def\LegendU{0.015}
\def\LegendV{0.985}
\def\LegendFont{\scriptsize}
% One line each. That only fits because the marker rows moved into their
% own box below: a single box carrying all four rows reached far enough
% across to bury the goal marker, which sits about 0.66 of the way across
% and 0.81 of the way up.
\def\LegendPressure{min pressure ($318$\,kPa, $700^{\circ}$, $6.29$\,m)}
\def\LegendDistance{min distance ($19{,}430$\,kPa, $1460^{\circ}$, $6.03$\,m)}
\def\LegendStart{start}
\def\LegendGoal{goal}

\newif\ifVineOverlayDone
\begingroup
\global\VineOverlayDonefalse
% ---------------------------------------------------------------------
% The obstacles of a dense course are drawn on the `vineobstacles` PGF
% layer, which is composited at \end{tikzpicture} -- that is, *after* the
% spy scope has closed. Anything on a layer is therefore invisible to the
% magnifier, and the inset would show enlarged paths over unenlarged
% obstacles. Flattening the layer puts the obstacles back in the picture
% body where spy re-executes them.
%
% The cost is that the obstacles now paint after the axis rather than
% beneath it. That is almost invisible here: a planned path lies in free
% space by construction, so obstacles only meet it where it grazes one.
% ---------------------------------------------------------------------
\def\pgfonlayer#1{}
\def\endpgfonlayer{}

\pgfplotsset{vine/course axis/.append style={
  height=\CourseHeight,
  tick label style={font=\CourseTickFont},
  label style={font=\CourseLabelFont},
}}

% Fractions of the axis box, which is the coordinate frame set up below.
\pgfmathsetmacro{\SpyU}{(\SpyX-\CourseXmin)/(\CourseXmax-\CourseXmin)}
\pgfmathsetmacro{\SpyV}{(\SpyY-\CourseYmin)/(\CourseYmax-\CourseYmin)}

\tikzset{
  every picture/.append style={
    % `spy using outlines` expands to `spy scope={<its own defaults>, #1}`,
    % so whatever is written inside these braces is applied last and wins.
    % That matters: the library hard-codes `very thin` on the small circle,
    % `thick` on the magnified one, and `\draw[thin]` on the connector that
    % `connect spies` installs. None of them can be reached from the options
    % of \spy itself, which is why setting `line width` there does nothing.
    % Setting `spy connection path` directly replaces what `connect spies`
    % would have done, so that key is no longer needed.
    spy using outlines={circle,
      magnification=\SpyMagnification,
      size=\SpySize,
      every spy on node/.append style={draw=\SpyColour, line width=\SpyLine},
      every spy in node/.append style={draw=\SpyColour, line width=\SpyLine,
                                       fill=white},
      spy connection path={\draw[\SpyColour, line width=\SpyLine]
                             (tikzspyonnode) -- (tikzspyinnode);}},
    % The inset has to be issued from inside the picture, since the axis
    % node it is positioned against is local to it. `execute at end
    % picture` runs after the figure's own content and before the picture
    % closes, which is exactly the window required.
    execute at end picture={%
      % pgfplots opens internal pictures of its own while laying the axis
      % out; the axis node does not exist in those, so skip them.
      \ifcsname pgf@sh@ns@vinecourseaxis\endcsname
      % The legend below is itself a pgfplots axis, which opens pictures of
      % its own; without this the hook would re-enter and recurse.
      \ifVineOverlayDone\else\global\VineOverlayDonetrue
        % Unit vectors taken from the axis corners, so (u,v) runs from
        % (0,0) at the bottom-left of the axis to (1,1) at the top-right
        % whatever the figure is scaled to.
        \begin{scope}[shift={(vinecourseaxis.south west)},
          x={($(vinecourseaxis.south east)-(vinecourseaxis.south west)$)},
          y={($(vinecourseaxis.north west)-(vinecourseaxis.south west)$)}]
          \coordinate (spy target) at (\SpyU,\SpyV);
          \coordinate (spy inset)  at (\SpyNodeU,\SpyNodeV);
          \coordinate (vine legend anchor) at (\LegendU,\LegendV);
        \end{scope}
        % The path and marker styles live in the /pgfplots/ key path, where
        % a plain \draw cannot reach them: \draw would look in /tikz/, and
        % the nested `vine/path` inside each of them fails there. Building
        % the legend from \addlegendimage inside a pgfplots axis sidesteps
        % that entirely -- the styles resolve exactly as they do for the
        % figure's own \addplot calls, so nothing is restated here.
        %
        % Each axis is one point across with its lines and ticks switched
        % off. `hide axis` would be the obvious way to do that, but it
        % suppresses the legend along with everything else.
        %
        % Two legends rather than one. The path rows are wide and the marker
        % rows are short, so a single box would carry a wide empty corner
        % over the figure. Splitting them and letting each box shrink to its
        % own contents leaves an L, which is the shape the entries actually
        % occupy. `legend columns=-1` then sets the markers side by side.
        \begin{axis}[scale only axis, width=1pt, height=1pt,
          at={(vine legend anchor)}, anchor=north west,
          axis line style={draw=none}, ticks=none,
          xmin=0, xmax=1, ymin=0, ymax=1,
          legend cell align=left,
          legend style={font=\LegendFont, at={(0,0)}, anchor=north west,
            name=vinepathlegend,
            fill=white, fill opacity=1, text opacity=1,
            draw=black!60, rounded corners=1pt, inner sep=1.5pt}]
          \addlegendimage{vine/pressure path}\addlegendentry{\LegendPressure}
          \addlegendimage{vine/distance path}\addlegendentry{\LegendDistance}
        \end{axis}
        % Anchored to the first legend's own node rather than to a fixed
        % offset. The boxes are sized by the font, in points, while the axis
        % is sized by the column, so any fraction-of-the-axis gap would drift
        % the moment either changed.
        \begin{axis}[scale only axis, width=1pt, height=1pt,
          at={(vinepathlegend.south west)}, anchor=north west,
          axis line style={draw=none}, ticks=none,
          xmin=0, xmax=1, ymin=0, ymax=1,
          legend cell align=left, legend columns=-1,
          legend style={font=\LegendFont, at={(0,0)}, anchor=north west,
            fill=white, fill opacity=1, text opacity=1,
            draw=black!60, rounded corners=1pt, inner sep=1.5pt,
            /tikz/every even column/.append style={column sep=6pt}}]
          \addlegendimage{only marks, mark=vinestart}\addlegendentry{\LegendStart}
          \addlegendimage{only marks, mark=vinetarget}\addlegendentry{\LegendGoal}
        \end{axis}

        \spy on (spy target) in node at (spy inset);

      \fi\fi
    },
  },
}

\input{figures/dense-maze}
\endgroup

%% file: figures/3D-maze-paths.tex
% colours matched to the hardware photograph
\definecolor{pathoptimal}{RGB}{82,217,94}%
\definecolor{paththird}{RGB}{82,138,217}%
\definecolor{pathfourth}{RGB}{224,138,63}%
\definecolor{pathfifth}{RGB}{219,70,75}%
\definecolor{vineobstacle}{RGB}{0,0,0}%
%% Four routes through the simulated 3x3x4 maze, ranked by growth pressure:
%% ranks 1 (green), 3 (orange), 4 (blue) and 6 (red).
%% Geometry from vine_planner (PathPlanner.compute_pressure_path), in inches.
%% Figure axes: x width, y depth, z height (up). Obstacles are planar and
%% the scene is emitted back-to-front by height, since pgfplots has no
%% depth buffer and draw order is the only depth cue available.
%%
%% Camera: azimuth 175 deg, elevation 12 deg, set through the axis
%% unit vectors so the 3:3:4 box proportions stay exact. The picture is
%% exactly \linewidth wide and 1.4118\linewidth tall.
%% CAMERA. Change these two numbers to re-aim the view; the unit vectors
%% below recompute themselves. Azimuth is measured from the -y axis and
%% elevation above the xy plane, both in degrees (MATLAB/pgfplots sense).
%% Keep elevation strictly positive: the walls and paths are emitted
%% back-to-front by height, and looking up from below would invert that.
\def\vineaz{175}%
\def\vineel{12}%
%% WIDTH. The picture is \vinescale x \linewidth wide. Set \vinescale
%% before \input to size it, e.g. to match a photograph beside it.
%% It must be a bare number -- it is multiplied in, not appended to a length.
\providecommand{\vinescale}{1}%
%% Projected width of the bounding box, so the picture is exactly \linewidth.
\pgfmathsetmacro{\vinespan}{18*abs(cos(\vineaz))+18*abs(sin(\vineaz))}%
\pgfmathsetmacro{\vinexx}{\vinescale*cos(\vineaz)/\vinespan}%
\pgfmathsetmacro{\vinexy}{-\vinescale*sin(\vineaz)*sin(\vineel)/\vinespan}%
\pgfmathsetmacro{\vineyx}{\vinescale*sin(\vineaz)/\vinespan}%
\pgfmathsetmacro{\vineyy}{\vinescale*cos(\vineaz)*sin(\vineel)/\vinespan}%
\pgfmathsetmacro{\vinezy}{\vinescale*cos(\vineel)/\vinespan}%
%% One knob for all four paths: change line width here.
\tikzset{vinepath/.style={line width=1.6pt, line join=round, line cap=round, mark=none}}%
\begin{tikzpicture}
\begin{axis}[
    scale only axis=true, enlargelimits=false,
    xmin=0, xmax=18, ymin=0, ymax=18, zmin=0, zmax=24,
    x={(\vinexx\linewidth,\vinexy\linewidth)},
    y={(\vineyx\linewidth,\vineyy\linewidth)},
    z={(0pt,\vinezy\linewidth)},
    hide axis, clip=false,
]
% course boundary, behind the scene (commented out)
% \draw[black, line width=1.0pt] (axis cs:0,0,0) -- (axis cs:18,0,0) -- (axis cs:18,18,0) -- (axis cs:0,18,0) -- cycle;
% \draw[black, line width=1.0pt] (axis cs:0,0,24) -- (axis cs:18,0,24) -- (axis cs:18,18,24) -- (axis cs:0,18,24) -- cycle;
% \draw[black, line width=1.0pt] (axis cs:0,0,0) -- (axis cs:0,0,24);
% \draw[black, line width=1.0pt] (axis cs:0,18,0) -- (axis cs:0,18,24);
% \draw[black, line width=1.0pt] (axis cs:18,0,0) -- (axis cs:18,0,24);
% \draw[black, line width=1.0pt] (axis cs:18,18,0) -- (axis cs:18,18,24);
% walls and paths, back to front
\addplot3[pathfifth, vinepath] coordinates {(18.0000,18.0000,0.0000) (4.1584,12.1188,6.0000)};
\addplot3[pathfourth, vinepath] coordinates {(18.0000,18.0000,0.0000) (5.9400,1.9800,6.0000)};
\addplot3[paththird, vinepath] coordinates {(18.0000,18.0000,0.0000) (5.9400,5.9400,6.0000)};
\addplot3[pathoptimal, vinepath] coordinates {(18.0000,18.0000,0.0000) (5.9400,16.0200,6.0000)};
\filldraw[vineobstacle, draw=white, line width=0.4pt, opacity=1] (axis cs:6,12,6) -- (axis cs:12,12,6) -- (axis cs:12,18,6) -- (axis cs:6,18,6) -- cycle;
\filldraw[vineobstacle, draw=white, line width=0.4pt, opacity=1] (axis cs:12,12,6) -- (axis cs:18,12,6) -- (axis cs:18,18,6) -- (axis cs:12,18,6) -- cycle;
\filldraw[vineobstacle, draw=white, line width=0.4pt, opacity=1] (axis cs:0,6,6) -- (axis cs:6,6,6) -- (axis cs:6,12,6) -- (axis cs:0,12,6) -- cycle;
\filldraw[vineobstacle, draw=white, line width=0.4pt, opacity=1] (axis cs:6,6,6) -- (axis cs:12,6,6) -- (axis cs:12,12,6) -- (axis cs:6,12,6) -- cycle;
\filldraw[vineobstacle, draw=white, line width=0.4pt, opacity=1] (axis cs:12,6,6) -- (axis cs:18,6,6) -- (axis cs:18,12,6) -- (axis cs:12,12,6) -- cycle;
\filldraw[vineobstacle, draw=white, line width=0.4pt, opacity=1] (axis cs:6,0,6) -- (axis cs:12,0,6) -- (axis cs:12,6,6) -- (axis cs:6,6,6) -- cycle;
\filldraw[vineobstacle, draw=white, line width=0.4pt, opacity=1] (axis cs:12,0,6) -- (axis cs:18,0,6) -- (axis cs:18,6,6) -- (axis cs:12,6,6) -- cycle;
\addplot3[pathfifth, vinepath] coordinates {(4.1584,12.1188,6.0000) (4.0200,12.0600,6.0600) (12.0600,4.0200,12.0000)};
\addplot3[pathfourth, vinepath] coordinates {(5.9400,1.9800,6.0000) (5.9400,1.9800,6.0600) (13.8996,5.9004,12.0000)};
\addplot3[paththird, vinepath] coordinates {(5.9400,5.9400,6.0000) (5.9400,5.9400,6.0600) (12.0600,12.0600,12.0000)};
\addplot3[pathoptimal, vinepath] coordinates {(5.9400,16.0200,6.0000) (5.9400,16.0200,6.0600) (12.0600,16.0200,12.0000)};
\filldraw[vineobstacle, draw=white, line width=0.4pt, opacity=1] (axis cs:0,12,12) -- (axis cs:6,12,12) -- (axis cs:6,18,12) -- (axis cs:0,18,12) -- cycle;
\filldraw[vineobstacle, draw=white, line width=0.4pt, opacity=1] (axis cs:6,12,12) -- (axis cs:12,12,12) -- (axis cs:12,18,12) -- (axis cs:6,18,12) -- cycle;
\filldraw[vineobstacle, draw=white, line width=0.4pt, opacity=1] (axis cs:0,6,12) -- (axis cs:6,6,12) -- (axis cs:6,12,12) -- (axis cs:0,12,12) -- cycle;
\filldraw[vineobstacle, draw=white, line width=0.4pt, opacity=1] (axis cs:6,6,12) -- (axis cs:12,6,12) -- (axis cs:12,12,12) -- (axis cs:6,12,12) -- cycle;
\filldraw[vineobstacle, draw=white, line width=0.4pt, opacity=1] (axis cs:12,6,12) -- (axis cs:18,6,12) -- (axis cs:18,12,12) -- (axis cs:12,12,12) -- cycle;
\filldraw[vineobstacle, draw=white, line width=0.4pt, opacity=1] (axis cs:0,0,12) -- (axis cs:6,0,12) -- (axis cs:6,6,12) -- (axis cs:0,6,12) -- cycle;
\filldraw[vineobstacle, draw=white, line width=0.4pt, opacity=1] (axis cs:6,0,12) -- (axis cs:12,0,12) -- (axis cs:12,6,12) -- (axis cs:6,6,12) -- cycle;
\addplot3[pathfifth, vinepath] coordinates {(12.0600,4.0200,12.0000) (16.0200,12.0600,18.0000)};
\addplot3[pathfourth, vinepath] coordinates {(13.8996,5.9004,12.0000) (13.9800,5.9400,12.0600) (16.0200,12.0600,18.0000)};
\addplot3[paththird, vinepath] coordinates {(12.0600,12.0600,12.0000) (12.0600,12.0600,12.0600) (12.0600,12.0600,18.0000)};
\addplot3[pathoptimal, vinepath] coordinates {(12.0600,16.0200,12.0000) (12.0600,12.0992,18.0000)};
\filldraw[vineobstacle, draw=white, line width=0.4pt, opacity=1] (axis cs:6,12,18) -- (axis cs:12,12,18) -- (axis cs:12,18,18) -- (axis cs:6,18,18) -- cycle;
\filldraw[vineobstacle, draw=white, line width=0.4pt, opacity=1] (axis cs:0,6,18) -- (axis cs:6,6,18) -- (axis cs:6,12,18) -- (axis cs:0,12,18) -- cycle;
\filldraw[vineobstacle, draw=white, line width=0.4pt, opacity=1] (axis cs:6,6,18) -- (axis cs:12,6,18) -- (axis cs:12,12,18) -- (axis cs:6,12,18) -- cycle;
\filldraw[vineobstacle, draw=white, line width=0.4pt, opacity=1] (axis cs:12,6,18) -- (axis cs:18,6,18) -- (axis cs:18,12,18) -- (axis cs:12,12,18) -- cycle;
\filldraw[vineobstacle, draw=white, line width=0.4pt, opacity=1] (axis cs:0,0,18) -- (axis cs:6,0,18) -- (axis cs:6,6,18) -- (axis cs:0,6,18) -- cycle;
\filldraw[vineobstacle, draw=white, line width=0.4pt, opacity=1] (axis cs:6,0,18) -- (axis cs:12,0,18) -- (axis cs:12,6,18) -- (axis cs:6,6,18) -- cycle;
\filldraw[vineobstacle, draw=white, line width=0.4pt, opacity=1] (axis cs:12,0,18) -- (axis cs:18,0,18) -- (axis cs:18,6,18) -- (axis cs:12,6,18) -- cycle;
\addplot3[pathfifth, vinepath] coordinates {(16.0200,12.0600,18.0000) (16.0200,12.0600,18.0600) (0.0000,0.0000,24.0000)};
\addplot3[pathfourth, vinepath] coordinates {(16.0200,12.0600,18.0000) (16.0200,12.0600,18.0600) (0.0000,0.0000,24.0000)};
\addplot3[paththird, vinepath] coordinates {(12.0600,12.0600,18.0000) (12.0600,12.0600,18.0600) (0.0000,0.0000,24.0000)};
\addplot3[pathoptimal, vinepath] coordinates {(12.0600,12.0992,18.0000) (12.0600,12.0600,18.0600) (0.0000,0.0000,24.0000)};
\end{axis}
\end{tikzpicture}

%% file: figures/3D-maze-pressure.tex
% colours matched to the hardware photograph
% legend order Best / 3rd / 4th / 6th -> green / blue / orange / red,
% the same order as the 2D pressure figure. Keep these identical to (a).
\definecolor{pathoptimal}{RGB}{82,217,94}%
\definecolor{paththird}{RGB}{82,138,217}%
\definecolor{pathfourth}{RGB}{224,138,63}%
\definecolor{pathfifth}{RGB}{219,70,75}%
%% Growth pressure against deployed length for the four plotted routes
%% through the 3x3x4 maze, in kPa and metres. Ranks 1, 3, 4 and 6.
%% Model constants: Y = 0.278 psi, lambda = 0.0341 psi/ft, T = 0.787 psi, mu = 0.315.
%% The legend sits outside the axis, to its right.
\begin{tikzpicture}
\begin{axis}[
    width=0.80\linewidth, height=0.54\linewidth,
    xlabel={Deployed length (m)}, ylabel={Predicted pressure (kPa)},
    xmin=0, xmax=1.60, ymin=6.5, ymax=38.5,
    xtick={0,0.5,1.0,1.5},
    ytick={10,15,20,25,30,35},
    tick align=outside, tick pos=left,
    tick label style={font=\scriptsize}, label style={font=\small},
    grid=both, grid style={draw=black!15, line width=0.3pt},
    reverse legend,
    legend cell align=left,
    legend style={font=\small, at={(1.04,1.0)}, anchor=north west,
                  draw=none, fill=none, row sep=2pt,
                  /tikz/every even column/.append style={column sep=3pt}},
]
\addplot[pathfifth, line width=1.8pt, mark=none] coordinates {(0.00000,7.3429) (0.41539,7.6633) (0.41539,11.7577) (0.74123,12.0090) (0.74123,18.3633) (1.01517,18.5746) (1.01517,24.6051) (1.01670,24.6062) (1.01670,35.9040) (1.54790,36.3137)};
\addlegendentry{4th}
\addplot[pathfourth, line width=1.8pt, mark=none] coordinates {(0.00000,7.3429) (0.53163,7.7530) (0.53163,10.6515) (0.53316,10.6526) (0.53316,13.8152) (0.80710,14.0265) (0.80710,16.6969) (1.02984,16.8687) (1.02984,21.3159) (1.03137,21.3171) (1.03137,30.9770) (1.56257,31.3868)};
\addlegendentry{3rd}
\addplot[paththird, line width=1.8pt, mark=none] coordinates {(0.00000,7.3429) (0.45923,7.6972) (0.45923,10.4393) (0.46076,10.4405) (0.46076,13.4842) (0.72739,13.6898) (0.72739,17.8938) (0.72891,17.8949) (0.88131,18.0125) (0.88131,25.6716) (1.34004,26.0255)};
\addlegendentry{2nd}
\addplot[pathoptimal, line width=1.8pt, mark=none] coordinates {(0.00000,7.3429) (0.34582,7.6097) (0.34582,10.0038) (0.34734,10.0050) (0.34734,12.3241) (0.56397,12.4912) (0.56397,16.1726) (0.74784,16.3145) (0.74784,20.8866) (1.20657,21.2405)};
\addlegendentry{Best}
\addplot[black!55, dashed, line width=1.0pt, mark=none, forget plot] coordinates {(0,21.5) (1.60,21.5)};
\node[anchor=south west, font=\scriptsize, black!65, inner sep=2pt] at (axis cs:0.03,21.5) {burst pressure};
\end{axis}
\end{tikzpicture}

%% file: figures/theorem-1-fig.tex
\begin{tikzpicture}[>=Stealth]

  %% ---------------- parameters ----------------
  \def\pa{0.7}       % path is the parabola y = -\pa*x^2
  \def\eps{1.5}       % radius of the epsilon-ball
  \def\tzero{-0.35}   % parameter of the ball centre on the path
  \def\tmin{-1.5}     % left end of drawn path
  \def\tmax{1.5}      % right end of drawn path

  \coordinate (X) at ({\tzero},{-\pa*\tzero*\tzero});

  %% ------- locate where the path crosses the ball -------
  \path[name path=curve,domain=\tmin:\tmax,samples=300,variable=\t]
       plot ({\t},{-\pa*\t*\t});
  \path[name path=ball] (X) circle[radius=\eps];
  \path[name intersections={of=curve and ball,by={XE,XL}}];

  %% ------- tangent directions at the two crossings -------
  \path let \p1=(XE), \p2=(XL) in \pgfextra{%
    \pgfmathsetmacro{\tmpE}{\x1/28.45274}%
    \pgfmathsetmacro{\tmpL}{\x2/28.45274}%
    \xdef\tEg{\tmpE}\xdef\tLg{\tmpL}%
  };
  \pgfmathsetmacro{\angE}{atan(-2*\pa*\tEg)}   % tangent bearing at x_enter
  \pgfmathsetmacro{\angL}{atan(-2*\pa*\tLg)}   % tangent bearing at x_exit
  \pgfmathanglebetweenpoints{\pgfpointanchor{XE}{center}}%
                            {\pgfpointanchor{XL}{center}}
  \pgfmathsetmacro{\angC}{\pgfmathresult>180 ? \pgfmathresult-360 : \pgfmathresult}

  \coordinate (TE) at ($(XE)+({1.15*cos(\angE)},{1.15*sin(\angE)})$);
  \coordinate (TL) at ($(XL)+({1.15*cos(\angL)},{1.15*sin(\angL)})$);
  \coordinate (CL) at ($(XL)!-1.15cm!(XE)$);

  %% ---------------- the path ----------------
  \draw[line width=1.5pt,dashed,domain=\tmin:\tmax,samples=300,variable=\t]
       plot ({\t},{-\pa*\t*\t});
  \begin{scope}
    \clip (X) circle[radius=\eps];
    \draw[line width=1.5pt,domain=\tmin:\tmax,samples=300,variable=\t]
         plot ({\t},{-\pa*\t*\t});
  \end{scope}

  %% ---------------- the ball ----------------
  \draw (X) circle[radius=\eps];
  \draw[->] (X) -- ++({\eps*cos(140)},{\eps*sin(140)});
  \node at ($(X)+({0.55*\eps*cos(140)+0.28*cos(50)},
                  {0.55*\eps*sin(140)+0.28*sin(50)})$) {$\epsilon$};

  %% ---------------- the shortcut ----------------
  \draw[red,line width=1.5pt] (XE) -- (XL)
       node[midway,below=1pt,black] {$h$};

  %% ------- tangent guides and turning angles -------
  \draw[gray,dotted] (XE) -- (TE);
  \draw[gray,dotted] (XL) -- (TL);
  \draw[gray,dotted] (XL) -- (CL);

  % theta_1: label pushed toward the chord so it clears the curve
  \pic[draw,angle radius=3mm] {angle=XL--XE--TE};
  \pgfmathsetmacro{\angOne}{\angC+0.28*(\angE-\angC)}
  \node at ($(XE)+({0.6*cos(\angOne)},{0.85*sin(\angOne)})$) {$\theta_1$};

  % theta_2: bisector is clear here, the curve bends away from the wedge
  \pic[draw,angle radius=3mm,angle eccentricity=2.0,
       "$\theta_2$"{inner sep=1pt}] {angle=TL--XL--CL};

  %% ---------------- labels ----------------
  \fill (X)  circle[radius=2.4pt];
  \fill (XE) circle[radius=2.4pt];
  \fill (XL) circle[radius=2.4pt];
  \node[above right=2pt and -1pt] at (X)  {$x$};
  \node[above left=-5pt and 0pt]   at (XE) {$x_1$};
  \node[above right=3pt and 4pt]   at (XL) {$x_2$};
\end{tikzpicture}

%% file: figures/theorem-2-fig.tex
\begin{tikzpicture}[>=Stealth]

  %% ---------------- parameters ----------------
  \def\angIn{35}      % bearing of the incoming segment
  \def\angOut{-35}    % bearing of the outgoing segment
  \def\eps{1.5}       % radius of the epsilon-ball
  \def\Lin{2.7}       % length of the incoming segment
  \def\Lout{2.4}      % length of the outgoing segment

  \coordinate (X)  at (0,0);
  \coordinate (S)  at ({\Lin*cos(\angIn+180)},{\Lin*sin(\angIn+180)});
  \coordinate (G)  at ({\Lout*cos(\angOut)},{\Lout*sin(\angOut)});
  \coordinate (XE) at ({\eps*cos(\angIn+180)},{\eps*sin(\angIn+180)});
  \coordinate (XL) at ({\eps*cos(\angOut)},{\eps*sin(\angOut)});

  %% guide points: extension of the incoming segment past X,
  %% and extension of the chord past x_exit
  \coordinate (EXT) at ({1.15*cos(\angIn)},{1.15*sin(\angIn)});
  \coordinate (CL)  at ($(XL)!-1.15cm!(XE)$);

  %% ---------------- the path ----------------
  \draw[line width=1.5pt,dashed] (S) -- (X) -- (G);
  \begin{scope}
    \clip (X) circle[radius=\eps];
    \draw[line width=1.5pt] (S) -- (X) -- (G);
  \end{scope}

  %% ---------------- the ball ----------------
  \draw (X) circle[radius=\eps];
  \draw[->] (X) -- ++({\eps*cos(140)},{\eps*sin(140)});
  \node at ($(X)+({0.55*\eps*cos(140)+0.28*cos(50)},
                  {0.55*\eps*sin(140)+0.28*sin(50)})$) {$\epsilon$};

  %% ---------------- the shortcut ----------------
  \draw[red,line width=1.5pt] (XE) -- (XL);

  %% ------- guides and turning angles -------
  \draw[gray,dotted] (X)  -- (EXT);
  \draw[gray,dotted] (XL) -- (CL);

  \pic[draw,angle radius=7mm,angle eccentricity=1.38,
       "$\theta_1$"{inner sep=1pt}] {angle=XL--XE--X};
  \pic[draw,angle radius=3mm,angle eccentricity=1.8,
       "$\theta_3$"{inner sep=1pt}] {angle=XL--X--EXT};
  \pic[draw,angle radius=7mm,angle eccentricity=1.45,
       "$\theta_2$"{inner sep=1pt}] {angle=G--XL--CL};

  %% ---------------- labels ----------------
  \fill (X)  circle[radius=2.4pt];
  \fill (XE) circle[radius=2.4pt];
  \fill (XL) circle[radius=2.4pt];
  \node[above=3pt]                at (X)  {$x$};
  \node[above left=-2pt and 0pt]  at (XE) {$x_1$};
  \node[above right=1pt and 5pt]  at (XL) {$x_2$};

\end{tikzpicture}